\documentclass{article}

\PassOptionsToPackage{numbers, compress}{natbib}
\usepackage[preprint]{neurips_2026}

\usepackage[utf8]{inputenc} 
\usepackage[T1]{fontenc}    
\usepackage{hyperref}       
\usepackage{url}            
\usepackage{booktabs}       
\usepackage{amsfonts}       
\usepackage{amsmath}
\usepackage{amssymb}
\usepackage{amsthm}
\usepackage{nicefrac}       
\usepackage{microtype}      
\usepackage{xcolor}         
\usepackage{multicol,multirow}

\theoremstyle{plain}
\newtheorem{theorem}{Theorem}[section]

\newtheorem{corollary}[theorem]{Corollary}
\theoremstyle{definition}

\theoremstyle{remark}

\usepackage{graphicx}
\usepackage{algorithm}
\usepackage{algorithmic}

\title{A Scalable Nonparametric Continuous-Time Survival Model through Numerical Quadrature}

\author{%
  Chaeyeon Lee \\
  Department of Statistics \\
  Ewha Womans University \\
  Seoul, Korea \\
  \And
  Sehwan Kim \\
  Department of Statistics \\
  Ewha Womans University \\
  Seoul, Korea \\
  \And
  Hyungrok Do \\
  Department of Population Health \\
  NYU Grossman School of Medicine \\
  New York, NY, USA \\
}

\begin{document}

\maketitle

\begin{abstract}
  Flexible continuous-time survival modeling is critical for capturing complex time-varying hazard dynamics in high-dimensional data; however, training such models remains challenging due to the intractable integral required for likelihood estimation. We introduce QSurv, a scalable deep learning framework that enables nonparametric continuous-time modeling without relying on time discretization or restrictive distributional assumptions. We propose a training objective based on Gauss-Legendre numerical quadrature, which approximates the cumulative hazard with high-order accuracy while facilitating efficient end-to-end training via standard backpropagation. Furthermore, to effectively capture non-stationary hazard dynamics in complex architectures, we introduce time-conditioned low-rank adaptation, a mechanism that conditions general neural backbones on time by dynamically modulating weights via low-rank updates. We provide theoretical analysis establishing approximation error bounds for cumulative-hazard evaluation. Comprehensive experiments across synthetic benchmarks, large-scale real-world tabular datasets, and high-dimensional medical imaging tasks demonstrate that QSurv achieves competitive predictive performance with advantages in instantaneous hazard function estimation, enabling more interpretable characterization of time-varying risk patterns.
\end{abstract}

\section{Introduction}

Survival analysis models the time to an event of interest and is fundamental to medical prognosis, engineering reliability, and financial risk stratification. Its primary challenge lies in handling censoring, where the event time is not fully observed for some subjects within the study period. In this work, we focus on right-censoring, which is the most prevalent form of censoring in practice. Classical methods such as the Kaplan-Meier (KM) estimator \citep{KM1958KM} and the Cox proportional hazards (PH) model \citep{cox1972coxph} are widely used in survival analysis. The KM estimator serves as a nonparametric approach to estimating the population-level survival function, whereas the Cox PH model provides a semiparametric framework linking covariates to the hazard via a log-linear predictor. While these methods remain standard tools owing to their simplicity and interpretability, they are often insufficient to capture the complex, high-dimensional interactions between covariates and time-to-event. Although machine learning approaches like Random Survival Forests (RSF) \citep{Hemant2008RSF} and classification-based frameworks such as survival stacking \citep{craig2021survivalstacking} offer improved risk prediction over classical methods, they remain limited in their ability to directly learn representations from unstructured and high-dimensional data, such as electronic health records and medical imaging, and can impose a substantial computational burden at scale.

To address these limitations, the field has increasingly turned to deep learning. One prominent line of work preserves the Cox framework while using a neural network to learn a nonlinear risk function (e.g., DeepSurv \citep{katzman2018deepsurv}). Another common direction adopts a discrete-time formulation, predicting hazards or event probabilities on a fixed time grid to enable flexible training objectives (e.g., DeepHit \citep{lee2018deephit}). Alternatively, parametric deep models support likelihood-based training but constrain the hazard through a chosen distribution family (e.g., Deep Survival Machines \citep{nagpal2021deep_dsm}). Notably, these approaches typically rely on structural assumptions about the survival function or hazard. Cox-based neural models primarily learn a relative-risk score under proportional hazards, discrete-time models represent risk on a prespecified time grid, and parametric models restrict the hazard through a chosen distributional family. Thus, their limitations are not only computational or predictive, but also concern how directly they characterize the continuous-time instantaneous hazard.

Direct characterization of the instantaneous hazard is clinically important because the hazard function is the most natural object for describing time-resolved risk. While the survival function summarizes accumulated risk up to time $t$, the instantaneous hazard $h(t | x)$ describes how risk changes locally among individuals who remain at risk. Many clinical settings exhibit such time-local structure, including delayed treatment effects, transient early toxicity, crossing hazards, late recurrence, and decaying or accelerating risks over follow-up. These patterns are difficult to summarize by a single proportional-hazards effect and may be obscured when models focus only on survival probabilities or fixed-grid event probabilities. Prior work has therefore emphasized both the limitations of routine hazard-ratio reporting under non-proportional hazards \citep{hernan2010hazards,stensrud2020test} and the practical value of examining hazard functions to understand disease dynamics and recurrence patterns \citep{royston2001flexible,hess2014getting}.

Despite this need, fully nonparametric deep learning for continuous-time hazard modeling remains underexplored because the survival likelihood requires both pointwise hazard evaluation and cumulative-hazard integration. Directly parameterizing $\lambda_\theta(t|x)$ with a flexible time-conditioned network is therefore computationally challenging. In our setting, the integrand is the modeled hazard $\lambda_\theta(t|x)$, so the quadrature error is controlled by its temporal smoothness. Existing deep survival models often avoid this difficulty through proportional-hazards, discrete-time, or parametric assumptions. More flexible continuous-time approaches have recently been proposed, but SODEN \citep{tang2022soden} depends on adaptive ODE solvers, while DeSurv \citep{danks2022desurv} models the cumulative distribution function rather than the hazard as the primitive quantity. Thus, scalable neural survival modeling that directly parameterizes the instantaneous hazard remains underdeveloped.

Beyond the integration challenge, direct hazard modeling also raises the question of how time should enter the network. Since $h(t | x)$ is a local-in-time risk function, the model must represent interactions between covariates $x$ and follow-up time $t$, rather than treating time as a passive index. A standard approach is to concatenate $t$ with $x$ at the input, which is often adequate for tabular data processed by multilayer perceptrons (MLPs); however, for more expressive backbones such as ResNet \citep{he2016deep}, DenseNet \citep{huang2017densely} and Transformers \citep{vaswani2017attention, dosovitskiy2020image}, simple input concatenation may be insufficient. A more effective strategy is to let temporal information modulate intermediate feature representations, enabling the model to capture time-varying risk more directly.

To address the gap between scalable deep learning and direct hazard modeling, we introduce Quadrature-based Nonparametric Continuous-Time Deep Survival Analysis (QSurv). QSurv parameterizes the instantaneous hazard $\lambda_\theta(t | x)$ as a flexible function of time and covariates, and uses numerical quadrature to evaluate the cumulative hazard required by the continuous-time likelihood. In this way, QSurv makes the hazard function the primary modeled object while avoiding the computational burden of adaptive ODE solvers and the precision loss of time discretization. Our main contributions are summarized as follows:

\begin{itemize}
    \item We propose a nonparametric continuous-time deep survival model that directly parameterizes the instantaneous hazard as a flexible function of time and covariates, enabling time-resolved characterization of risk without fixed time discretization.
    \item We introduce a likelihood objective based on numerical quadrature. This eliminates the need for ODE solvers, enabling efficient end-to-end training with standard backpropagation and minibatch stochastic gradient descent.
    \item We move beyond simple input concatenation by proposing time-conditioned low-rank approximation (Time-LoRA), which conditions general neural network backbones on $t$ by dynamically modulating network weights, resulting in richer temporal representations.
    \item We perform comprehensive experiments across synthetic benchmarks, real-world tabular data, and high-dimensional medical imaging tasks. Our results confirm QSurv's efficacy and feasibility compared to state-of-the-art baselines.
\end{itemize}

\section{Related Works}\label{sec:2}

Integrating deep learning approaches into survival analysis has led to a wide range of methodological advances. While these deep learning-based survival analysis methods can be categorized in several ways, to keep the discussion focused, we organize the literature into two broad settings: (1) discrete-time models and (2) continuous-time models. For a comprehensive review, see \citep{Wiegrebe2024review} and the references therein.
    
In discrete-time survival models, time is partitioned into a set of intervals, and the model outputs interval-specific hazards or a discrete probability distribution over discretized time points. Multi-Task Logistic Regression (MTLR) \citep{yu2011learning} casts survival estimation as a sequence of binary classification tasks, one for each time interval, using logistic regression. This framework has been extended to neural MTLR (N-MTLR) to capture nonlinear effects \citep{fotso2018deep}. DeepHit \citep{lee2018deephit} similarly models discrete-time event probabilities with a neural network and accommodates competing risks by incorporating an additional ranking-based loss. In contrast, Nnet-survival estimates the conditional hazard probabilities in discrete time points \citep{gensheimer2019scalable}.

For continuous-time survival models, existing deep learning approaches can be broadly grouped into three directions.

\textbf{Cox models:} Many methods start from the Cox proportional hazards (PH) framework and relax parts of the original formulation to improve flexibility. A widely used example is DeepSurv, which replaces the linear predictor in the Cox model with a neural network \citep{katzman2018deepsurv}. Several extensions have been proposed. CoxCC modifies the loss to mimic a case–control sampling scheme, and CoxTime allows the hazard to depend on time through a time-varying effect model \citep{kvamme2019time}. Deep Cox Mixtures further extends DeepSurv by modeling heterogeneous subpopulations via a mixture formulation \citep{nagpal2021deep_dcm}. More generally, the Deep Extended Hazard model provides a flexible hazard formulation that generalizes the baseline hazard structure \citep{zhong2021deep}. 

\textbf{Parametric survival models:} Another line of work adopts fully parametric specifications, where the survival time is modeled as a function of covariates and error term, typically assuming distributions such as Weibull or log-normal. Neural networks are then used to learn covariate-dependent parameters \citep{Ranganath2016deepsurv,Bennis2020wei,ava2020lognormal,nagpal2021deep_dsm}. On the other hand, Survival Mixture Density Networks model the event-time density using a Gaussian mixture and use a neural network to map covariates to the mixture parameters \citep{han2022survival}.

\textbf{ODE-based models:} More recently, ODE-based approaches have been proposed that model the evolution of the cumulative hazard (or related quantities) via an ordinary differential equation \citep{tang2022soden,tang2023survival}. This yields a principled continuous-time formulation capable of capturing complex temporal dynamics. A key limitation, however, is that these models can be computationally very expensive to train, and they also introduce additional hyperparameters related to ODE-solver. A related approach, DeSurv \citep{danks2022desurv}, formulates an ODE for the cumulative distribution function and discretizes it via Gauss–Legendre quadrature, but recovers the event density via time differentiation rather than directly modeling the hazard.

The literature highlights key limitations that motivate our work. Existing methodologies often impose structural assumptions on the hazard or event-time distribution: discrete-time models depend on a chosen time grid, Cox models inherit restrictions from the PH formulation, and parametric models rely on specific distributional assumptions. While ODE-based formulations offer flexibility in modeling hazard functions, they incur significant computational costs, and the existing quadrature-based alternative bypasses the hazard altogether by modeling the cumulative distribution function. This leaves a clear gap: scalable, fully nonparametric deep learning that directly models the instantaneous hazard remains underexplored. Yet the instantaneous hazard is itself a clinically meaningful quantity, providing a direct measure of risk at any given time that informs time-sensitive decisions such as patient monitoring and intervention timing.

\section{Problem Formulation}

Let $T, C \in \mathbb{R}_+$ denote the random variables for the true time-to-event and censoring time, respectively. We consider a setting with covariates $X \in \mathcal{X} \subseteq \mathbb{R}^d$. Under the assumption of covariate-dependent censoring, we assume that the event time and censoring time are conditionally independent given the covariates, i.e., $T \perp C | X$. In the observational setting, the true event time $T$ is not always observed directly due to censoring. Instead, we observe the pair $(O, \Delta)$, where $O = \min(T, C)$ denotes the observed time and $\Delta = \mathbb{I}(T \leq C)$ is the binary event indicator. Given a dataset of $n$ i.i.d. realizations $\mathcal{D} = \{(x_i, o_i, \delta_i)\}_{i=1}^{n}$, our primary objective is to learn the parameters $\theta$ of the conditional hazard function $\lambda_\theta(t|x)$. More broadly, however, survival analysis aims to characterize the full time-to-event distribution, which may involve estimating the survival function $S_\theta(t|x)$, the cumulative hazard $\Lambda_\theta(t|x)$, or discrete-time event probabilities.

The hazard function $\lambda_\theta(t|x)$ fully characterizes the time-to-event distribution. It is linked to the cumulative hazard $\Lambda_\theta(t|x)$ and the survival function $S_\theta(t|x)$ through the following equations:
\begin{align}
    \Lambda_\theta(t|x) = \int_0^t \lambda_\theta(u|x) du, \qquad S_\theta(t|x) = \exp\left(-\Lambda_\theta(t|x)\right).
\end{align}
To estimate the model parameters $\theta$, we minimize the negative log-likelihood (NLL). Under the standard assumption of covariate-dependent censoring ($T \perp C | X$), the terms governing the censoring distribution are independent of $\theta$ and factor out of the optimization. This yields the standard survival analysis objective function:
\begin{equation}\label{eqn:surv-nll}
    \mathcal{L}(\theta) = -\frac{1}{n}\sum_{i=1}^n \left[ \delta_i \log \lambda_\theta(o_i|x_i) - \Lambda_\theta(o_i|x_i) \right].
\end{equation}
Minimizing this objective with respect to $\theta$ requires a modeling approach that is both expressive enough to capture complex temporal dynamics and computationally tractable for evaluating the cumulative hazard integral. Existing literature navigates this trade-off through various approximations. This objective exposes the central computational difficulty in direct neural network parameterized hazard modeling: the event contribution requires evaluating the hazard locally at $o_i$, whereas the survival contribution requires integrating the same learned hazard over the full interval $[0,o_i]$. QSurv is designed to preserve the hazard as the primitive modeled quantity while making this cumulative-hazard evaluation efficient and differentiable.


\subsection{Quadrature-based Nonparametric Continuous-Time Deep Survival Analysis (QSurv)}
We model the log-hazard with a network $f_\theta:\mathcal{X}\times\mathbb{R}_+\to\mathbb{R}$ and set $\lambda_\theta(t|x)=\exp\{f_\theta(x,t)\}$. Minimizing the negative log-likelihood requires evaluating the cumulative hazard $\Lambda_\theta(o_i | x_i) = \int_0^{o_i} \lambda_\theta(s | x_i) \, ds$. While the log-hazard requires only a single forward pass, the cumulative hazard is analytically intractable due to the lack of parametric structure. To resolve this, we approximate $\Lambda_\theta(t | x)$ using $K$-point Gauss-Legendre quadrature. By applying the change of variables $s = t\tau$ to map the integration domain $[0, t]$ to the canonical interval $[0, 1]$, we obtain the approximation:
\begin{equation} \label{eqn:quadrature_approx}
\begin{aligned}
    \hat{\Lambda}_\theta(t | x) = \frac{t}{2} \sum_{k=1}^K w_k\, \lambda_\theta(t\tau_k | x),
\end{aligned}
\end{equation}
where $\tau_k = (\xi_k + 1)/2$ are the Gauss-Legendre nodes shifted from the standard interval $[-1, 1]$ to $[0, 1]$ via the affine map, and $\{w_k\}_{k=1}^K$ are the corresponding standard Gauss-Legendre weights satisfying $\sum_k w_k = 2$. The factor $t/2$ accounts for the Jacobian of the change of variables $s = t\tau$ combined with the $[-1, 1] \to [0, 1]$ affine transformation.

Unlike standard fixed-grid approximations, such as Riemann sums or the trapezoidal rule, which are constrained to equidistant points, Gauss-Legendre quadrature optimizes node locations to maximize precision. The standard nodes are derived from the roots of the $K$-th degree Legendre polynomial $P_K(\xi)$ and are mapped to our target domain $[0, 1]$ via the affine transformation $\tau_{k} = (\xi_{k} + 1)/2$. The resulting rule is exact for polynomials of degree up to $2K-1$, and therefore achieves high accuracy for sufficiently smooth hazards with relatively few evaluations. Consequently, the approximation error is small for smooth hazard functions even with a small $K$. Crucially, this approximation is implemented as a vectorized set of $K$ network evaluations that remains fully differentiable with respect to $\theta$, enabling end-to-end training via standard backpropagation with a time complexity of $\mathcal{O}(K)$ per cumulative hazard evaluation.

Importantly, compared to these Riemann or trapezoidal approximations, Gauss-Legendre quadrature is not tied to a dense time discretization. While grid-based methods require a large number of points to minimize discretization error, scaling computational cost linearly with the desired resolution, Gauss-Legendre leverages optimal node placement to achieve superior accuracy with a fixed, low budget of $K$ evaluations. The parameters of QSurv are learned by minimizing the negative log-likelihood of the observed data. Given a dataset the objective function is defined as:
\begin{equation}
    \mathcal{L}(\theta;\mathcal{D}) = -\frac{1}{n}\sum_{i=1}^{n} \Big[ \delta_i f_{\theta}(x_i, o_i) - \frac{o_i}{2} \sum_{k=1}^{K} w_k \exp f_{\theta}(x_i, o_i \tau_k ) \Big],
\end{equation}
where the first term captures the instantaneous risk for observed events, and the second term, approximated via the quadrature method described above. Algorithm \ref{alg:training} describes the training procedure for QSurv.

\begin{algorithm}[t]
\caption{Training procedure for QSurv}
\label{alg:training}
\begin{algorithmic}[1]
\STATE \textbf{Require}: Dataset $\mathcal{D} = \{(x_i, o_i, \delta_i)\}_{i=1}^{n}$, Gauss-Legendre nodes $\boldsymbol{\tau} \in [0,1]^K$ and weights $\mathbf{w} \in \mathbb{R}^K$.
\STATE \textbf{Initialize} network parameters $\theta$.
\WHILE{not converged}
    \STATE Sample minibatch $\mathcal{B} \subset \mathcal{D}$ of size $B$.
    \FORALL{$(x_i, o_i, \delta_i) \in \mathcal{B}$ \textbf{in parallel}}
        \STATE Compute log-hazard by feed-forward: $f_\theta(x_i, o_i)$
        \STATE Estimate cumulative hazard (vectorized): $\hat{\Lambda}_i \leftarrow \tfrac{o_i}{2} \cdot \mathbf{w}^\top \exp\!\big(f_\theta(x_i, o_i \boldsymbol{\tau})\big)$.
    \ENDFOR
    \STATE Compute loss: $\mathcal{L}(\theta; \mathcal{B}) \leftarrow -\frac{1}{|\mathcal{B}|} \sum_{(x_i,o_i,\delta_i)\in\mathcal{B}} \left[ \delta_i f_\theta(x_i,o_i)-\hat{\Lambda}_i \right]$.
    \STATE Update parameters: $\theta \leftarrow \theta - \eta \nabla_\theta \mathcal{L}(\theta; \mathcal{B})$
\ENDWHILE
\end{algorithmic}
\end{algorithm}

\subsection{Approximation Error Bound}
The validity of QSurv relies on the convergence of the quadrature approximation to the exact cumulative-hazard integral induced by the modeled hazard. This follows from classical properties of Gauss-Legendre quadrature: the $K$-point rule is exact for polynomials of degree up to $2K-1$ and admits a high-order error bound for sufficiently smooth integrands \citep{golub1969calculation,davis2007methods}. In our setting, the integrand is the neural network parameterized hazard $\lambda_\theta(t | x)$, so the quadrature error is controlled by the temporal smoothness of the learned hazard. We state the resulting bound below and defer the proof, which is a direct application of the standard Gauss-Legendre error formula after an affine transformation from $[-1,1]$ to $[0,t]$, to Appendix~\ref{appendix:proofs}.

\begin{theorem}[Approximation Error Bound of Cumulative Hazard Function]\label{thm:error_bound}
Let the hazard function $\lambda_\theta(\cdot|x)$ be $2K$-times continuously differentiable on $[0, t_{\max}]$. Let $\Lambda_\theta(t|x)$ denote the exact cumulative hazard induced by $\lambda_\theta$ and $\hat{\Lambda}_{\theta, K}(t|x)$ be the approximation computed using $K$-point Gauss-Legendre quadrature. Then, for any $t \in [0,t_{\max}]$, the approximation error is bounded by:
\begin{equation}
\begin{aligned}
    \varepsilon_K(t|x) = \Big| \Lambda_\theta(t|x) - \hat{\Lambda}_{\theta, K}(t|x) \Big| \leq \frac{t^{2K+1} (K!)^4}{(2K+1)[(2K)!]^3} \max_{\tau \in [0,t]} \Big| \lambda_\theta^{(2K)}(\tau|x) \Big|.    
\end{aligned}
\end{equation}
Furthermore, if $\lambda_\theta(\cdot|x)$ is a polynomial of degree at most $2K-1$ with respect to time, the approximation is exact, i.e., $\hat{\Lambda}_{\theta, K}(t|x) = \Lambda_\theta(t|x)$.\end{theorem}

Theorem~\ref{thm:error_bound} shows that the quadrature error is controlled by the number of nodes $K$ and the temporal smoothness of the learned hazard. For sufficiently smooth hazards, the factorial term in the bound yields rapid error decay as $K$ increases, and the Gauss--Legendre rule is exact for hazards that are polynomial in time of degree at most $2K-1$. Thus, a moderate fixed value of $K$ can provide accurate cumulative-hazard estimates in practice, as confirmed empirically in Appendix~\ref{appendix:num_nodes}. The result also motivates using smooth hazard parameterizations, such as networks with Softplus or GELU activations, since high-order quadrature accuracy relies on the existence of temporal derivatives. Overall, $K$ should be viewed as a numerical precision parameter rather than a model-complexity or regularization parameter.

We establish a direct link between the numerical precision of the integration scheme and the training objective. Specifically, we show that the error in the negative log-likelihood is bounded by the error in the cumulative hazard approximation.

\begin{corollary}[Propagation of Approximation Error]\label{cor:error_prop}
    Let $\mathcal{L}(\theta; t, \delta, x) = -\delta \log \lambda_\theta(t|x) + \Lambda_\theta(t|x)$ be the true negative log-likelihood for a sample with time $t$, event indicator $\delta$, and features $x$. Let $\hat{\mathcal{L}}_K$ denote the approximate loss computed using the quadrature estimate $\hat{\Lambda}_{\theta,K}$. Then the loss approximation error equals the cumulative hazard approximation error:
    \begin{equation}
    \big| \mathcal{L}(\theta; t, \delta, x) - \hat{\mathcal{L}}_K(\theta; t, \delta, x) \big| = \varepsilon_K(t|x).
    \end{equation}
\end{corollary}

This result relates the stability of the learning process to the integration accuracy. Because the cumulative hazard enters the negative log-likelihood linearly, the approximation error is additive and does not compound through non-linear transformations. Consequently, the bound $\varepsilon_K(t|x)$ translates directly to the loss function. When combined with the convergence properties established in Theorem \ref{thm:error_bound}, this suggests that controlling the quadrature error effectively limits the deviation of the training objective from the true likelihood, allowing the approximate loss to serve as a reliable proxy during optimization.

\subsection{Choice of the Number of Quadrature Nodes}

The number of quadrature nodes $K$ controls the numerical accuracy of the cumulative-hazard approximation and the computational cost of training. Importantly, $K$ is not a time-discretization parameter in the same sense as the number of bins in discrete-time survival models. QSurv still represents the hazard as a continuous function of time; $K$ only determines how accurately the integral of this hazard is approximated inside the likelihood. Very small values of $K$ may provide an inaccurate approximation when the hazard varies rapidly over time, which can distort the training objective. Increasing $K$ improves the fidelity of the quadrature approximation but requires additional evaluations of the neural network parameterized hazard at the quadrature nodes. Thus, $K$ should be interpreted as a numerical precision parameter rather than a regularization hyperparameter: reducing $K$ does not impose a simpler hazard model, but instead coarsens the approximation to the cumulative hazard. In practice, we use a moderate fixed value of $K$ and examine sensitivity to this choice empirically in Section \ref{appendix:num_nodes}.

\subsection{Time-Conditioned LoRA}
For tabular data, a common strategy to model time-varying hazards is to treat time $t$ as an additional input feature, simply concatenating it with the static covariates $x$, e.g., as in CoxTime \citep{kvamme2019time}, DeSurv \citep{danks2022desurv}, and SODEN \citep{tang2022soden}. However, this approach does not trivially extend to high-dimensional data such as medical images processed by backbones like ResNet \citep{he2016deep}, DenseNet \citep{huang2017densely} and Transformers \citep{vaswani2017attention, dosovitskiy2020image}. Concatenating a scalar $t$ directly to an image input is not naturally compatible with convolutional or transformer backbones.  To address this, we introduce time-conditioned LoRA, which uses a time-dependent weight matrix at the network's penultimate layer to make the hazard prediction adaptive in $t$ without re-running the backbone. This avoids $K$ redundant forward passes of the heavy backbone, substantially reducing the computational cost compared with fully time-entangled architectures.

For this penultimate layer with weight $W \in \mathbb{R}^{d_{\text{out}} \times d_{\text{in}}}$ and bias $b \in \mathbb{R}^{d_{\text{out}}}$, we parameterize the time-varying weight $W(t)$ as a low-rank perturbation of the static base \cite{hu2022lora}:
\begin{equation}
    W(t) = W + \Delta W(t) = W + U \cdot \text{diag}(s(t)) \cdot V,
\end{equation}
where $U \in \mathbb{R}^{d_{\text{out}} \times r}$ and $V \in \mathbb{R}^{r \times d_{\text{in}}}$ are learnable low-rank projection matrices ($r \ll \min(d_{\text{in}}, d_{\text{out}})$). The temporal dynamics are captured by the modulation coefficients $s(t) = g_\psi(\phi(t)) \in \mathbb{R}^r$, which are generated by a lightweight MLP $g_\psi$ acting on a time embedding $\phi(t)$ \cite{kazemi2019time2vec}.

The forward pass can be re-interpreted as a time-gated residual branch operating within a learned subspace:
\begin{equation}
    z(t) = Wh + U \big( s(t) \odot (Vh) \big) + b.
\end{equation}
Here, $V$ projects the input $h$ into a latent subspace, $s(t)$ applies a time-varying activation, and $U$ projects the result back to the output dimension; a final time-independent linear layer then maps $z(t)$ to the scalar log-hazard $f_\theta(x,t)$. This structural decoupling yields a significant computational advantage during numerical integration, where the hazard must be evaluated at multiple quadrature nodes $\{\tau_{k}\}_{k=1}^K$ for a single subject. Because Time-LoRA is applied only at the penultimate layer, the embedding $h$ is independent of $t$, so the base computation $Wh$ and the initial projection $Vh$ depend solely on the static input $x$ and are computed only once and cached. For each quadrature node $\tau_k$, only the lightweight temporal branch is evaluated to obtain $s(o_i \tau_k)$, followed by the low-rank recombination. This avoids $K$ redundant forward passes of the heavy backbone, drastically reducing the inference cost compared to fully time-entangled architectures. This caching is only valid because the modulation is confined to the head: applying Time-LoRA to an intermediate backbone layer would make the inputs to all downstream layers time-dependent and require a full backbone pass per quadrature node, which is why we restrict it to the penultimate layer. Unlike FiLM, which applies element-wise affine modulation, Time-LoRA allows time-dependent low-rank mixing across feature channels, providing a more expressive head for modeling time-varying hazards \cite{dumoulin2018feature}.

\section{Experiments}

We performed a series of experiments across a comprehensive set of benchmark datasets to evaluate the proposed methods. To ensure a rigorous comparison across different modeling paradigms, we evaluated against Cox-style deep learning models CoxCC and CoxTime \citep{kvamme2019time}, discrete-time models DeepHit \citep{lee2018deephit} and NnetSurv \citep{gensheimer2019scalable}, and continuous-time estimators including Survival MDN \citep{han2022survival}, DeSurv \citep{danks2022desurv}, and SODEN \citep{tang2022soden}. For a controlled comparison of time-conditioned continuous-time architectures, CoxTime, DeSurv, and SODEN were implemented with the same time-conditioned LoRA mechanism used in QSurv. Details of the deep survival models are provided in Appendix \ref{appendix:models}. Code for reproducing the models and experiments is available at \texttt{https://github.com/hyungrok-do/qsurv}.

\subsection{Simulation Study}

\begin{table*}[!t]
    \small
    \centering
    \setlength{\tabcolsep}{4pt}
    \caption{Comparison of $L_1$ error of instantaneous hazard (smaller the better) across different parametric distributions with standard deviations over 20 random seeds. Best results are \textbf{bold}, second best are \underline{underlined}. Full results are presented in Table \ref{tab:simulation_full} in the appendix.}
    \label{tab:simulation}
    \resizebox{\linewidth}{!}{%
    \begin{tabular}{cccccccc}
    \toprule
    \textbf{Distribution} & \textbf{CoxCC} & \textbf{CoxTime} & \textbf{NnetSurv} & \textbf{MDN} & \textbf{DeSurv} & \textbf{SODEN} & \textbf{QSurv} \\
    \midrule
    Exponential  & 0.1775 $\pm$ 0.0410 & 0.1775 $\pm$ 0.0422 & 0.3830 $\pm$ 0.0438 & 0.0568 $\pm$ 0.0446 & 0.0448 $\pm$ 0.0159 & \textbf{0.0330 $\pm$ 0.0138} & \underline{0.0352 $\pm$ 0.0164} \\
    Weibull      & 0.1576 $\pm$ 0.0322 & 0.1182 $\pm$ 0.0207 & 0.2383 $\pm$ 0.0254 & 0.0644 $\pm$ 0.0102 & 0.0295 $\pm$ 0.0090 & \textbf{0.0272 $\pm$ 0.0119} & \underline{0.0282 $\pm$ 0.0139} \\
    Gamma        & 0.9012 $\pm$ 0.0921 & 0.5661 $\pm$ 0.1769 & 0.7249 $\pm$ 0.0370 & 0.5646 $\pm$ 0.0459 & 0.2930 $\pm$ 0.0907 & \textbf{0.1716 $\pm$ 0.0901} & \underline{0.1718 $\pm$ 0.0859} \\
    Gompertz     & 0.1713 $\pm$ 0.0262 & 0.1890 $\pm$ 0.0549 & 0.2921 $\pm$ 0.0276 & 0.0920 $\pm$ 0.0140 & 0.0502 $\pm$ 0.0278 & \underline{0.0262 $\pm$ 0.0174} & \textbf{0.0258 $\pm$ 0.0180} \\
    Log-normal   & 0.1235 $\pm$ 0.0248 & 0.1114 $\pm$ 0.0294 & 0.1938 $\pm$ 0.0219 & 0.1882 $\pm$ 0.0520 & 0.0473 $\pm$ 0.0201 & \underline{0.0462 $\pm$ 0.0183} & \textbf{0.0457 $\pm$ 0.0185} \\
    Log-logistic & 0.1621 $\pm$ 0.0210 & 0.1266 $\pm$ 0.0242 & 0.3290 $\pm$ 0.0525 & 0.1031 $\pm$ 0.0090 & \underline{0.0513 $\pm$ 0.0299} & 0.0517 $\pm$ 0.0295 & \textbf{0.0491 $\pm$ 0.0291} \\
    \bottomrule
    \end{tabular}
    }\vspace{-1.5em}
\end{table*}

We conducted a simulation study to assess whether QSurv can recover instantaneous hazard functions under nonlinear covariate effects. Details of the data-generating mechanisms and additional results are provided in Appendix \ref{appendix:simulation}. Table \ref{tab:simulation} reports $L_1$ error on the test set for the estimated instantaneous hazard functions. QSurv and SODEN consistently produced the smallest hazard-estimation errors, particularly for distributions with more complex hazard shapes. The main difference was computational: QSurv achieved similar accuracy with much shorter training time. Across simulation settings, SODEN required approximately 60 seconds on average, compared with approximately 6 seconds for QSurv. Figures \ref{fig:sim_exponential}--\ref{fig:sim_loglogistic} show the corresponding test-set estimates of the instantaneous hazard, cumulative hazard, and survival functions.

\subsection{Real-world Datasets}

\begin{table*}[!t]
\centering
\caption{Performance of QSurv vs deep survival baselines across the full and quantile-based horizons. Best results are \textbf{bold}, second best are \underline{underlined}. The right-most D-cal column is the mean $p$-value of the discrete-time D-calibration goodness-of-fit test (higher is better; $p>0.05$ indicates no calibration violation).}
\label{tab:realdata}
\resizebox{\linewidth}{!}{%
\begin{tabular}{ll ccc ccc ccc c}
\toprule
 & & \multicolumn{3}{c}{\textbf{Horizon: Full ($\hat{G}(\tau){\approx}0.001$)}} & \multicolumn{3}{c}{\textbf{Horizon: $T_{Q_1}$}} & \multicolumn{3}{c}{\textbf{Horizon: $T_{Q_2}$ (median)}} & \\
\cmidrule(lr){3-5} \cmidrule(lr){6-8} \cmidrule(lr){9-11}
Dataset & Model & $C_{td} \uparrow$ & IBS $\downarrow$ & IBLL $\uparrow$ & $C_{td} \uparrow$ & IBS $\downarrow$ & IBLL $\uparrow$ & $C_{td} \uparrow$ & IBS $\downarrow$ & IBLL $\uparrow$ & D-cal $\uparrow$ \\
\midrule \midrule
\multirow{8}{*}{COVID-19-NY}
 & CoxCC & 0.5693 $\pm$ 0.0961 & 0.1885 $\pm$ 0.0882 & -0.5211 $\pm$ 0.2118 & 0.5925 $\pm$ 0.1640 & 0.0048 $\pm$ 0.0052 & -0.0309 $\pm$ 0.0349 & 0.5504 $\pm$ 0.0633 & 0.0366 $\pm$ 0.0090 & -0.1656 $\pm$ 0.0339 & 0.7264 $\pm$ 0.4338 \\
 & CoxTime & 0.5587 $\pm$ 0.0847 & 0.2079 $\pm$ 0.1211 & -0.6197 $\pm$ 0.3588 & 0.5901 $\pm$ 0.1031 & \textbf{0.0044 $\pm$ 0.0046} & -0.0306 $\pm$ 0.0337 & 0.5636 $\pm$ 0.0891 & 0.0334 $\pm$ 0.0100 & -0.1592 $\pm$ 0.0327 & 0.6061 $\pm$ 0.5335 \\
 & DeepHit & 0.6677 $\pm$ 0.0793 & \underline{0.0708 $\pm$ 0.0209} & -0.2426 $\pm$ 0.0600 & 0.4951 $\pm$ 0.1224 & \underline{0.0045 $\pm$ 0.0046} & \underline{-0.0291 $\pm$ 0.0282} & 0.6516 $\pm$ 0.0859 & 0.0371 $\pm$ 0.0130 & -0.1660 $\pm$ 0.0491 & 0.7470 $\pm$ 0.3440 \\
 & NnetSurv & 0.6660 $\pm$ 0.0844 & 0.0797 $\pm$ 0.0159 & -0.2629 $\pm$ 0.0406 & 0.5753 $\pm$ 0.0990 & 0.0045 $\pm$ 0.0046 & \textbf{-0.0265 $\pm$ 0.0259} & 0.6793 $\pm$ 0.1175 & 0.0344 $\pm$ 0.0069 & -0.1463 $\pm$ 0.0171 & 0.8787 $\pm$ 0.2473 \\
 & MDN & 0.6889 $\pm$ 0.0626 & 0.1892 $\pm$ 0.0797 & -0.5626 $\pm$ 0.2426 & 0.5669 $\pm$ 0.1987 & 0.0185 $\pm$ 0.0100 & -0.0954 $\pm$ 0.0358 & 0.7116 $\pm$ 0.0536 & 0.0406 $\pm$ 0.0107 & -0.1660 $\pm$ 0.0323 & 0.7658 $\pm$ 0.2906 \\
 & DeSurv & \underline{0.7420 $\pm$ 0.0923} & \textbf{0.0617 $\pm$ 0.0269} & \textbf{-0.1981 $\pm$ 0.0802} & 0.7173 $\pm$ 0.2020 & 0.0052 $\pm$ 0.0049 & -0.0367 $\pm$ 0.0258 & \underline{0.7944 $\pm$ 0.0615} & \textbf{0.0322 $\pm$ 0.0089} & \underline{-0.1410 $\pm$ 0.0334} & 0.8502 $\pm$ 0.1983 \\
 & SODEN & 0.7143 $\pm$ 0.0530 & 0.1000 $\pm$ 0.0282 & -0.2967 $\pm$ 0.0709 & \textbf{0.7356 $\pm$ 0.1669} & 0.0048 $\pm$ 0.0049 & -0.0335 $\pm$ 0.0275 & 0.7689 $\pm$ 0.1099 & \underline{0.0324 $\pm$ 0.0098} & \textbf{-0.1360 $\pm$ 0.0357} & 0.9221 $\pm$ 0.1042 \\
 & QSurv & \textbf{0.7456 $\pm$ 0.1046} & 0.0716 $\pm$ 0.0155 & \underline{-0.2318 $\pm$ 0.0489} & \underline{0.7260 $\pm$ 0.1002} & 0.0062 $\pm$ 0.0050 & -0.0405 $\pm$ 0.0259 & \textbf{0.7951 $\pm$ 0.0908} & 0.0352 $\pm$ 0.0157 & -0.1441 $\pm$ 0.0488 & 0.7391 $\pm$ 0.2436 \\
\midrule
\multirow{8}{*}{C4KC-KiTS}
 & CoxCC & 0.4116 $\pm$ 0.1798 & 0.1099 $\pm$ 0.0050 & -0.3644 $\pm$ 0.0240 & 0.4834 $\pm$ 0.2542 & 0.0307 $\pm$ 0.0196 & -0.1523 $\pm$ 0.0805 & 0.4030 $\pm$ 0.1679 & \underline{0.0740 $\pm$ 0.0132} & \underline{-0.2840 $\pm$ 0.0468} & 1.0000 $\pm$ 0.0001 \\
 & CoxTime & \underline{0.5808 $\pm$ 0.1985} & \underline{0.1086 $\pm$ 0.0015} & \underline{-0.3571 $\pm$ 0.0139} & 0.5877 $\pm$ 0.2812 & \underline{0.0304 $\pm$ 0.0194} & -0.1479 $\pm$ 0.0788 & \underline{0.5806 $\pm$ 0.1851} & \textbf{0.0734 $\pm$ 0.0124} & \textbf{-0.2787 $\pm$ 0.0408} & 0.9999 $\pm$ 0.0001 \\
 & DeepHit & 0.5090 $\pm$ 0.1089 & 0.2179 $\pm$ 0.1236 & -0.7631 $\pm$ 0.3839 & 0.5290 $\pm$ 0.0644 & 0.0310 $\pm$ 0.0195 & -0.1570 $\pm$ 0.0838 & 0.4900 $\pm$ 0.1264 & 0.0820 $\pm$ 0.0192 & -0.3271 $\pm$ 0.0690 & 0.6826 $\pm$ 0.4021 \\
 & NnetSurv & 0.5422 $\pm$ 0.1375 & 0.1262 $\pm$ 0.0240 & -0.5230 $\pm$ 0.1907 & \underline{0.7345 $\pm$ 0.1797} & 0.0340 $\pm$ 0.0172 & -0.1964 $\pm$ 0.1066 & 0.5371 $\pm$ 0.1172 & 0.0867 $\pm$ 0.0277 & -0.4008 $\pm$ 0.1529 & 0.9574 $\pm$ 0.0529 \\
 & MDN & 0.5230 $\pm$ 0.2206 & 0.1106 $\pm$ 0.0139 & -0.4742 $\pm$ 0.2106 & 0.5290 $\pm$ 0.1920 & 0.0352 $\pm$ 0.0186 & -0.1908 $\pm$ 0.1246 & 0.5270 $\pm$ 0.2244 & 0.0773 $\pm$ 0.0155 & -0.3829 $\pm$ 0.2099 & 0.9693 $\pm$ 0.0473 \\
 & DeSurv & 0.5740 $\pm$ 0.1630 & 0.1156 $\pm$ 0.0152 & -0.6108 $\pm$ 0.3646 & 0.6274 $\pm$ 0.2445 & 0.0329 $\pm$ 0.0232 & -0.2136 $\pm$ 0.1838 & 0.5685 $\pm$ 0.1636 & 0.0826 $\pm$ 0.0228 & -0.4910 $\pm$ 0.2947 & 0.9889 $\pm$ 0.0225 \\
 & SODEN & 0.4987 $\pm$ 0.1159 & \textbf{0.1074 $\pm$ 0.0045} & \textbf{-0.3558 $\pm$ 0.0307} & 0.5679 $\pm$ 0.1812 & 0.0307 $\pm$ 0.0196 & \underline{-0.1401 $\pm$ 0.0765} & 0.5150 $\pm$ 0.0939 & 0.0754 $\pm$ 0.0131 & -0.2879 $\pm$ 0.0463 & 0.9995 $\pm$ 0.0004 \\
 & QSurv & \textbf{0.6381 $\pm$ 0.1446} & 0.1260 $\pm$ 0.0453 & -0.4002 $\pm$ 0.1113 & \textbf{0.7404 $\pm$ 0.1547} & \textbf{0.0288 $\pm$ 0.0178} & \textbf{-0.1220 $\pm$ 0.0679} & \textbf{0.6207 $\pm$ 0.1596} & 0.0776 $\pm$ 0.0235 & -0.2910 $\pm$ 0.0644 & 0.9845 $\pm$ 0.0209 \\
\midrule
\multirow{8}{*}{BraTS}
 & CoxCC & 0.4925 $\pm$ 0.0344 & 0.2021 $\pm$ 0.1203 & -0.6222 $\pm$ 0.3545 & 0.4976 $\pm$ 0.0676 & 0.1414 $\pm$ 0.0227 & -0.5501 $\pm$ 0.2447 & 0.4949 $\pm$ 0.0382 & 0.2168 $\pm$ 0.0411 & -0.7373 $\pm$ 0.3262 & 0.0702 $\pm$ 0.1141 \\
 & CoxTime & 0.4541 $\pm$ 0.0359 & 0.1171 $\pm$ 0.0130 & -0.3764 $\pm$ 0.0363 & 0.4414 $\pm$ 0.0549 & \underline{0.1355 $\pm$ 0.0128} & -0.4456 $\pm$ 0.0388 & 0.4460 $\pm$ 0.0406 & 0.2021 $\pm$ 0.0127 & -0.5939 $\pm$ 0.0319 & 0.3090 $\pm$ 0.3068 \\
 & DeepHit & 0.5196 $\pm$ 0.0383 & 0.1164 $\pm$ 0.0175 & -0.3765 $\pm$ 0.0453 & 0.5020 $\pm$ 0.0679 & 0.1357 $\pm$ 0.0125 & -0.4582 $\pm$ 0.0378 & 0.5130 $\pm$ 0.0457 & \underline{0.2010 $\pm$ 0.0160} & -0.6005 $\pm$ 0.0461 & 0.1414 $\pm$ 0.1793 \\
 & NnetSurv & 0.5355 $\pm$ 0.0395 & 0.1243 $\pm$ 0.0394 & -0.5086 $\pm$ 0.3472 & 0.5371 $\pm$ 0.0533 & 0.1806 $\pm$ 0.1093 & -0.5546 $\pm$ 0.2757 & 0.5471 $\pm$ 0.0465 & 0.2599 $\pm$ 0.1342 & -0.8321 $\pm$ 0.5395 & 0.3879 $\pm$ 0.5029 \\
 & MDN & 0.5333 $\pm$ 0.0537 & 0.1274 $\pm$ 0.0103 & -0.4131 $\pm$ 0.0252 & 0.5321 $\pm$ 0.0342 & 0.1754 $\pm$ 0.0082 & -0.5387 $\pm$ 0.0175 & 0.5429 $\pm$ 0.0479 & 0.2120 $\pm$ 0.0088 & -0.6149 $\pm$ 0.0186 & 0.0000 $\pm$ 0.0000 \\
 & DeSurv & 0.5459 $\pm$ 0.0297 & \textbf{0.1038 $\pm$ 0.0176} & \underline{-0.3405 $\pm$ 0.0707} & 0.5570 $\pm$ 0.0372 & \textbf{0.1355 $\pm$ 0.0141} & \textbf{-0.4378 $\pm$ 0.0363} & 0.5438 $\pm$ 0.0297 & 0.2044 $\pm$ 0.0195 & \underline{-0.5930 $\pm$ 0.0451} & 0.1535 $\pm$ 0.1623 \\
 & SODEN & \underline{0.5672 $\pm$ 0.0510} & 0.1089 $\pm$ 0.0121 & -0.3489 $\pm$ 0.0442 & \underline{0.5984 $\pm$ 0.0894} & 0.1511 $\pm$ 0.0308 & -0.4757 $\pm$ 0.0717 & \underline{0.5763 $\pm$ 0.0607} & 0.2217 $\pm$ 0.0373 & -0.6358 $\pm$ 0.0900 & 0.0951 $\pm$ 0.1047 \\
 & QSurv & \textbf{0.5811 $\pm$ 0.0387} & \underline{0.1038 $\pm$ 0.0133} & \textbf{-0.3327 $\pm$ 0.0371} & \textbf{0.6149 $\pm$ 0.0559} & 0.1380 $\pm$ 0.0254 & \underline{-0.4393 $\pm$ 0.0628} & \textbf{0.5968 $\pm$ 0.0537} & \textbf{0.1993 $\pm$ 0.0329} & \textbf{-0.5868 $\pm$ 0.0868} & 0.2161 $\pm$ 0.2395 \\
\midrule \midrule
\multirow{7}{*}{FLCHAIN}
 & CoxCC & 0.7967 $\pm$ 0.0101 & 0.0936 $\pm$ 0.0024 & -0.3103 $\pm$ 0.0075 & 0.8024 $\pm$ 0.0120 & 0.0680 $\pm$ 0.0023 & -0.2395 $\pm$ 0.0075 & 0.7988 $\pm$ 0.0099 & \underline{0.0861 $\pm$ 0.0023} & -0.2914 $\pm$ 0.0076 & 0.9224 $\pm$ 0.1189 \\
 & CoxTime & \underline{0.7968 $\pm$ 0.0101} & 0.0936 $\pm$ 0.0024 & \underline{-0.3100 $\pm$ 0.0074} & 0.8027 $\pm$ 0.0125 & \underline{0.0680 $\pm$ 0.0023} & \underline{-0.2389 $\pm$ 0.0076} & \underline{0.7989 $\pm$ 0.0102} & 0.0862 $\pm$ 0.0023 & \underline{-0.2910 $\pm$ 0.0074} & 0.9354 $\pm$ 0.1325 \\
 & DeepHit & 0.7883 $\pm$ 0.0107 & 0.1056 $\pm$ 0.0026 & -0.3400 $\pm$ 0.0069 & 0.7891 $\pm$ 0.0145 & 0.0755 $\pm$ 0.0029 & -0.2587 $\pm$ 0.0081 & 0.7896 $\pm$ 0.0114 & 0.0974 $\pm$ 0.0028 & -0.3190 $\pm$ 0.0076 & 0.0192 $\pm$ 0.0482 \\
 & NnetSurv & 0.7887 $\pm$ 0.0113 & \underline{0.0935 $\pm$ 0.0028} & -0.3103 $\pm$ 0.0086 & 0.7898 $\pm$ 0.0152 & 0.0681 $\pm$ 0.0026 & -0.2402 $\pm$ 0.0082 & 0.7903 $\pm$ 0.0121 & 0.0864 $\pm$ 0.0026 & -0.2922 $\pm$ 0.0084 & 0.7866 $\pm$ 0.2870 \\
 & MDN & 0.7954 $\pm$ 0.0096 & 0.0946 $\pm$ 0.0020 & -0.3142 $\pm$ 0.0068 & 0.8012 $\pm$ 0.0117 & 0.0689 $\pm$ 0.0021 & -0.2425 $\pm$ 0.0069 & 0.7973 $\pm$ 0.0097 & 0.0872 $\pm$ 0.0020 & -0.2949 $\pm$ 0.0071 & 0.4600 $\pm$ 0.2720 \\
 & DeSurv & 0.7968 $\pm$ 0.0097 & 0.0936 $\pm$ 0.0021 & -0.3106 $\pm$ 0.0067 & \underline{0.8030 $\pm$ 0.0116} & 0.0680 $\pm$ 0.0023 & -0.2398 $\pm$ 0.0073 & 0.7988 $\pm$ 0.0097 & 0.0862 $\pm$ 0.0022 & -0.2919 $\pm$ 0.0070 & 0.8706 $\pm$ 0.1812 \\
 & QSurv & \textbf{0.7968 $\pm$ 0.0096} & \textbf{0.0935 $\pm$ 0.0023} & \textbf{-0.3097 $\pm$ 0.0070} & \textbf{0.8031 $\pm$ 0.0118} & \textbf{0.0678 $\pm$ 0.0023} & \textbf{-0.2387 $\pm$ 0.0077} & \textbf{0.7989 $\pm$ 0.0094} & \textbf{0.0860 $\pm$ 0.0022} & \textbf{-0.2907 $\pm$ 0.0073} & 0.8819 $\pm$ 0.1852 \\
\midrule
\multirow{7}{*}{FRAMINGHAM}
 & CoxCC & 0.7407 $\pm$ 0.0142 & \underline{0.0882 $\pm$ 0.0029} & \underline{-0.2929 $\pm$ 0.0088} & 0.7556 $\pm$ 0.0131 & \textbf{0.0607 $\pm$ 0.0032} & -0.2162 $\pm$ 0.0099 & 0.7407 $\pm$ 0.0142 & \underline{0.0882 $\pm$ 0.0029} & \underline{-0.2929 $\pm$ 0.0088} & 0.9139 $\pm$ 0.1385 \\
 & CoxTime & \underline{0.7411 $\pm$ 0.0156} & \textbf{0.0881 $\pm$ 0.0031} & \textbf{-0.2925 $\pm$ 0.0095} & 0.7561 $\pm$ 0.0146 & 0.0608 $\pm$ 0.0034 & \textbf{-0.2160 $\pm$ 0.0109} & \underline{0.7411 $\pm$ 0.0156} & \textbf{0.0881 $\pm$ 0.0031} & \textbf{-0.2925 $\pm$ 0.0095} & 0.9543 $\pm$ 0.0693 \\
 & DeepHit & 0.7262 $\pm$ 0.0136 & 0.1780 $\pm$ 0.0162 & -0.5220 $\pm$ 0.0412 & 0.7403 $\pm$ 0.0136 & 0.0819 $\pm$ 0.0089 & -0.2948 $\pm$ 0.0254 & 0.7262 $\pm$ 0.0136 & 0.1780 $\pm$ 0.0162 & -0.5220 $\pm$ 0.0412 & 0.0000 $\pm$ 0.0000 \\
 & NnetSurv & 0.7316 $\pm$ 0.0165 & 0.0902 $\pm$ 0.0033 & -0.2997 $\pm$ 0.0110 & 0.7439 $\pm$ 0.0159 & 0.0615 $\pm$ 0.0031 & -0.2200 $\pm$ 0.0103 & 0.7316 $\pm$ 0.0165 & 0.0902 $\pm$ 0.0033 & -0.2997 $\pm$ 0.0110 & 0.6171 $\pm$ 0.3546 \\
 & MDN & 0.7378 $\pm$ 0.0128 & 0.0890 $\pm$ 0.0025 & -0.2960 $\pm$ 0.0084 & 0.7541 $\pm$ 0.0125 & 0.0611 $\pm$ 0.0031 & -0.2183 $\pm$ 0.0108 & 0.7378 $\pm$ 0.0128 & 0.0890 $\pm$ 0.0025 & -0.2960 $\pm$ 0.0084 & 0.8712 $\pm$ 0.2163 \\
 & DeSurv & 0.7400 $\pm$ 0.0138 & 0.0887 $\pm$ 0.0030 & -0.2954 $\pm$ 0.0104 & \underline{0.7565 $\pm$ 0.0137} & 0.0610 $\pm$ 0.0035 & -0.2178 $\pm$ 0.0116 & 0.7400 $\pm$ 0.0138 & 0.0887 $\pm$ 0.0030 & -0.2954 $\pm$ 0.0104 & 0.8984 $\pm$ 0.1424 \\
 & QSurv & \textbf{0.7414 $\pm$ 0.0146} & 0.0884 $\pm$ 0.0029 & -0.2933 $\pm$ 0.0087 & \textbf{0.7578 $\pm$ 0.0145} & \underline{0.0607 $\pm$ 0.0033} & \underline{-0.2161 $\pm$ 0.0105} & \textbf{0.7414 $\pm$ 0.0146} & 0.0884 $\pm$ 0.0029 & -0.2933 $\pm$ 0.0087 & 0.9004 $\pm$ 0.1399 \\
\midrule
\multirow{7}{*}{GBSG}
 & CoxCC & 0.6554 $\pm$ 0.0294 & \textbf{0.1759 $\pm$ 0.0117} & \textbf{-0.5220 $\pm$ 0.0328} & 0.6697 $\pm$ 0.0430 & \underline{0.0866 $\pm$ 0.0088} & \underline{-0.2991 $\pm$ 0.0284} & 0.6678 $\pm$ 0.0291 & 0.1408 $\pm$ 0.0139 & -0.4386 $\pm$ 0.0366 & 0.7995 $\pm$ 0.2822 \\
 & CoxTime & 0.6495 $\pm$ 0.0294 & 0.1822 $\pm$ 0.0176 & -0.5385 $\pm$ 0.0464 & 0.6683 $\pm$ 0.0425 & 0.0877 $\pm$ 0.0084 & -0.3027 $\pm$ 0.0258 & 0.6656 $\pm$ 0.0293 & 0.1433 $\pm$ 0.0142 & -0.4473 $\pm$ 0.0371 & 0.7181 $\pm$ 0.3025 \\
 & DeepHit & 0.6336 $\pm$ 0.0280 & 0.2071 $\pm$ 0.0301 & -0.5975 $\pm$ 0.0793 & 0.6497 $\pm$ 0.0401 & 0.0903 $\pm$ 0.0087 & -0.3162 $\pm$ 0.0270 & 0.6531 $\pm$ 0.0314 & 0.1556 $\pm$ 0.0185 & -0.4759 $\pm$ 0.0436 & 0.4138 $\pm$ 0.4239 \\
 & NnetSurv & 0.6245 $\pm$ 0.0314 & 0.2001 $\pm$ 0.0782 & -0.7529 $\pm$ 0.9511 & 0.6401 $\pm$ 0.0423 & 0.1172 $\pm$ 0.1302 & -0.4358 $\pm$ 0.5667 & 0.6388 $\pm$ 0.0301 & 0.1724 $\pm$ 0.1255 & -0.6716 $\pm$ 0.9790 & 0.6485 $\pm$ 0.3673 \\
 & MDN & 0.6529 $\pm$ 0.0286 & 0.1818 $\pm$ 0.0108 & -0.5406 $\pm$ 0.0321 & 0.6666 $\pm$ 0.0484 & 0.0916 $\pm$ 0.0093 & -0.3187 $\pm$ 0.0309 & 0.6649 $\pm$ 0.0286 & 0.1442 $\pm$ 0.0124 & -0.4537 $\pm$ 0.0334 & 0.4328 $\pm$ 0.3027 \\
 & DeSurv & \textbf{0.6592 $\pm$ 0.0245} & \underline{0.1769 $\pm$ 0.0120} & -0.5248 $\pm$ 0.0344 & \textbf{0.6784 $\pm$ 0.0364} & 0.0873 $\pm$ 0.0074 & -0.3029 $\pm$ 0.0203 & \textbf{0.6727 $\pm$ 0.0252} & \underline{0.1403 $\pm$ 0.0113} & \underline{-0.4378 $\pm$ 0.0280} & 0.3674 $\pm$ 0.2050 \\
 & QSurv & \underline{0.6558 $\pm$ 0.0297} & 0.1780 $\pm$ 0.0129 & \underline{-0.5247 $\pm$ 0.0353} & \underline{0.6765 $\pm$ 0.0453} & \textbf{0.0865 $\pm$ 0.0076} & \textbf{-0.2983 $\pm$ 0.0219} & \underline{0.6697 $\pm$ 0.0299} & \textbf{0.1402 $\pm$ 0.0134} & \textbf{-0.4369 $\pm$ 0.0343} & 0.6001 $\pm$ 0.3362 \\
\midrule
\multirow{7}{*}{METABRIC}
 & CoxCC & 0.6346 $\pm$ 0.0200 & \underline{0.1549 $\pm$ 0.0099} & -0.4679 $\pm$ 0.0294 & 0.6229 $\pm$ 0.0322 & 0.1055 $\pm$ 0.0050 & -0.3561 $\pm$ 0.0157 & 0.6368 $\pm$ 0.0260 & 0.1562 $\pm$ 0.0055 & -0.4798 $\pm$ 0.0158 & 0.6223 $\pm$ 0.2984 \\
 & CoxTime & \underline{0.6565 $\pm$ 0.0229} & 0.1578 $\pm$ 0.0102 & -0.4739 $\pm$ 0.0304 & \textbf{0.6739 $\pm$ 0.0320} & \underline{0.1032 $\pm$ 0.0063} & \underline{-0.3460 $\pm$ 0.0230} & \textbf{0.6637 $\pm$ 0.0246} & \underline{0.1533 $\pm$ 0.0070} & -0.4707 $\pm$ 0.0244 & 0.5473 $\pm$ 0.3588 \\
 & DeepHit & \textbf{0.6626 $\pm$ 0.0182} & 0.1609 $\pm$ 0.0102 & -0.4844 $\pm$ 0.0281 & 0.6669 $\pm$ 0.0300 & 0.1040 $\pm$ 0.0046 & -0.3494 $\pm$ 0.0131 & \underline{0.6608 $\pm$ 0.0232} & 0.1571 $\pm$ 0.0039 & -0.4803 $\pm$ 0.0100 & 0.5653 $\pm$ 0.2963 \\
 & NnetSurv & 0.6459 $\pm$ 0.0232 & 0.1626 $\pm$ 0.0101 & -0.4870 $\pm$ 0.0263 & 0.6531 $\pm$ 0.0312 & 0.1044 $\pm$ 0.0055 & -0.3505 $\pm$ 0.0170 & 0.6505 $\pm$ 0.0259 & 0.1565 $\pm$ 0.0058 & -0.4787 $\pm$ 0.0164 & 0.3262 $\pm$ 0.3765 \\
 & MDN & 0.6365 $\pm$ 0.0178 & 0.1708 $\pm$ 0.0089 & -0.5143 $\pm$ 0.0231 & 0.6309 $\pm$ 0.0334 & 0.1147 $\pm$ 0.0052 & -0.3852 $\pm$ 0.0142 & 0.6413 $\pm$ 0.0223 & 0.1600 $\pm$ 0.0048 & -0.4911 $\pm$ 0.0129 & 0.0000 $\pm$ 0.0000 \\
 & DeSurv & 0.6551 $\pm$ 0.0207 & 0.1554 $\pm$ 0.0110 & \underline{-0.4656 $\pm$ 0.0302} & 0.6660 $\pm$ 0.0343 & 0.1033 $\pm$ 0.0039 & -0.3460 $\pm$ 0.0123 & 0.6599 $\pm$ 0.0261 & 0.1533 $\pm$ 0.0052 & \underline{-0.4699 $\pm$ 0.0140} & 0.3543 $\pm$ 0.2764 \\
 & QSurv & 0.6542 $\pm$ 0.0198 & \textbf{0.1548 $\pm$ 0.0100} & \textbf{-0.4638 $\pm$ 0.0290} & \underline{0.6683 $\pm$ 0.0313} & \textbf{0.1031 $\pm$ 0.0049} & \textbf{-0.3451 $\pm$ 0.0153} & 0.6602 $\pm$ 0.0220 & \textbf{0.1531 $\pm$ 0.0048} & \textbf{-0.4687 $\pm$ 0.0137} & 0.3580 $\pm$ 0.2688 \\
\midrule
\multirow{7}{*}{NWTCO}
 & CoxCC & 0.6422 $\pm$ 0.0345 & 0.1126 $\pm$ 0.0039 & -0.3688 $\pm$ 0.0141 & 0.6343 $\pm$ 0.0347 & 0.0800 $\pm$ 0.0031 & -0.2871 $\pm$ 0.0107 & 0.6394 $\pm$ 0.0335 & 0.1023 $\pm$ 0.0029 & -0.3503 $\pm$ 0.0102 & 0.9777 $\pm$ 0.0471 \\
 & CoxTime & 0.6377 $\pm$ 0.0347 & 0.1139 $\pm$ 0.0047 & -0.3740 $\pm$ 0.0198 & 0.6310 $\pm$ 0.0346 & 0.0810 $\pm$ 0.0038 & -0.2904 $\pm$ 0.0136 & 0.6350 $\pm$ 0.0342 & 0.1036 $\pm$ 0.0038 & -0.3549 $\pm$ 0.0156 & 0.9341 $\pm$ 0.2133 \\
 & DeepHit & 0.6360 $\pm$ 0.0266 & 0.1147 $\pm$ 0.0023 & -0.3744 $\pm$ 0.0108 & 0.6282 $\pm$ 0.0263 & 0.0807 $\pm$ 0.0025 & -0.2891 $\pm$ 0.0078 & 0.6342 $\pm$ 0.0259 & 0.1033 $\pm$ 0.0018 & -0.3527 $\pm$ 0.0077 & 0.8889 $\pm$ 0.2197 \\
 & NnetSurv & 0.6301 $\pm$ 0.0199 & 0.1155 $\pm$ 0.0027 & -0.3771 $\pm$ 0.0113 & 0.6236 $\pm$ 0.0209 & 0.0817 $\pm$ 0.0030 & -0.2930 $\pm$ 0.0080 & 0.6284 $\pm$ 0.0197 & 0.1046 $\pm$ 0.0023 & -0.3570 $\pm$ 0.0070 & 0.7297 $\pm$ 0.3042 \\
 & MDN & \textbf{0.6532 $\pm$ 0.0273} & \textbf{0.1102 $\pm$ 0.0032} & \textbf{-0.3622 $\pm$ 0.0117} & \textbf{0.6448 $\pm$ 0.0269} & 0.0808 $\pm$ 0.0028 & -0.2938 $\pm$ 0.0089 & \textbf{0.6498 $\pm$ 0.0263} & 0.1032 $\pm$ 0.0019 & -0.3565 $\pm$ 0.0080 & 0.9426 $\pm$ 0.0962 \\
 & DeSurv & \underline{0.6492 $\pm$ 0.0236} & 0.1120 $\pm$ 0.0030 & \underline{-0.3664 $\pm$ 0.0115} & \underline{0.6398 $\pm$ 0.0227} & \underline{0.0798 $\pm$ 0.0030} & \underline{-0.2862 $\pm$ 0.0088} & \underline{0.6459 $\pm$ 0.0231} & \underline{0.1020 $\pm$ 0.0028} & \underline{-0.3489 $\pm$ 0.0095} & 0.9934 $\pm$ 0.0155 \\
 & QSurv & 0.6480 $\pm$ 0.0209 & \underline{0.1118 $\pm$ 0.0027} & -0.3667 $\pm$ 0.0112 & 0.6398 $\pm$ 0.0204 & \textbf{0.0793 $\pm$ 0.0027} & \textbf{-0.2848 $\pm$ 0.0075} & 0.6449 $\pm$ 0.0202 & \textbf{0.1015 $\pm$ 0.0021} & \textbf{-0.3477 $\pm$ 0.0075} & 0.9401 $\pm$ 0.1839 \\
\midrule
\multirow{7}{*}{SUPPORT2}
 & CoxCC & 0.6770 $\pm$ 0.0061 & 0.1798 $\pm$ 0.0046 & -0.5449 $\pm$ 0.0201 & 0.7122 $\pm$ 0.0079 & 0.1202 $\pm$ 0.0022 & -0.3896 $\pm$ 0.0065 & 0.6897 $\pm$ 0.0057 & 0.1927 $\pm$ 0.0031 & -0.5714 $\pm$ 0.0106 & 0.0781 $\pm$ 0.0903 \\
 & CoxTime & \underline{0.6941 $\pm$ 0.0052} & \underline{0.1778 $\pm$ 0.0045} & -0.5336 $\pm$ 0.0140 & \underline{0.7508 $\pm$ 0.0092} & \underline{0.1147 $\pm$ 0.0023} & \underline{-0.3680 $\pm$ 0.0071} & 0.7061 $\pm$ 0.0052 & \underline{0.1884 $\pm$ 0.0031} & -0.5573 $\pm$ 0.0081 & 0.1508 $\pm$ 0.2010 \\
 & DeepHit & 0.6847 $\pm$ 0.0063 & 0.1914 $\pm$ 0.0041 & -0.5653 $\pm$ 0.0102 & 0.7082 $\pm$ 0.0083 & 0.1238 $\pm$ 0.0021 & -0.3954 $\pm$ 0.0055 & 0.6875 $\pm$ 0.0070 & 0.2090 $\pm$ 0.0020 & -0.6042 $\pm$ 0.0045 & 0.0657 $\pm$ 0.0990 \\
 & NnetSurv & 0.6722 $\pm$ 0.0058 & \textbf{0.1773 $\pm$ 0.0043} & \textbf{-0.5293 $\pm$ 0.0117} & 0.7115 $\pm$ 0.0060 & \textbf{0.1137 $\pm$ 0.0017} & \textbf{-0.3650 $\pm$ 0.0049} & 0.6827 $\pm$ 0.0059 & \textbf{0.1883 $\pm$ 0.0030} & \underline{-0.5565 $\pm$ 0.0081} & 0.3449 $\pm$ 0.2590 \\
 & MDN & 0.6746 $\pm$ 0.0057 & 0.1886 $\pm$ 0.0041 & -0.5564 $\pm$ 0.0106 & 0.7017 $\pm$ 0.0123 & 0.1370 $\pm$ 0.0023 & -0.4407 $\pm$ 0.0060 & 0.6853 $\pm$ 0.0075 & 0.2138 $\pm$ 0.0022 & -0.6146 $\pm$ 0.0055 & 0.0000 $\pm$ 0.0000 \\
 & DeSurv & \textbf{0.6950 $\pm$ 0.0048} & 0.1792 $\pm$ 0.0053 & -0.5357 $\pm$ 0.0157 & 0.7501 $\pm$ 0.0094 & 0.1163 $\pm$ 0.0032 & -0.3723 $\pm$ 0.0092 & \underline{0.7067 $\pm$ 0.0055} & 0.1897 $\pm$ 0.0035 & -0.5597 $\pm$ 0.0083 & 0.0160 $\pm$ 0.0420 \\
 & QSurv & 0.6936 $\pm$ 0.0063 & 0.1778 $\pm$ 0.0042 & \underline{-0.5308 $\pm$ 0.0110} & \textbf{0.7512 $\pm$ 0.0077} & 0.1154 $\pm$ 0.0028 & -0.3697 $\pm$ 0.0080 & \textbf{0.7067 $\pm$ 0.0059} & 0.1885 $\pm$ 0.0032 & \textbf{-0.5562 $\pm$ 0.0075} & 0.0583 $\pm$ 0.1195 \\
\bottomrule \bottomrule
\end{tabular}
}\vspace{-2em}
\end{table*}

We evaluated QSurv on nine benchmark datasets, comprising six tabular datasets (FLCHAIN, FRAMINGHAM, GBSG, METABRIC, NWTCO, SUPPORT2) and three radiology imaging datasets (COVID-19-NY, C4KC-KiTS, and BraTS). These datasets cover diverse medical domains and vary substantially in sample size, feature dimensionality, and censoring rates; their main characteristics are summarized in Appendix~\ref{appendix:experiments}. Model performance was evaluated using the time-dependent concordance index ($C_{td}$), Integrated Brier Score (IBS), Integrated Binomial Log-Likelihood (IBLL), and D-calibration. Higher $C_{td}$ and IBLL and lower IBS indicate better performance. We define these metrics in Appendix~\ref{appendix:metrics}. To assess performance across follow-up durations, we evaluated each method over three horizons: the full horizon defined by $\hat{G}(\tau)=0.001$, and two restricted quantile-based horizons, $T_{Q_1}$ and $T_{Q_2}$, with $T_{Q_2}$ corresponding to the median horizon. For robust comparison, we performed random-search hyperparameter optimization with 30 configurations per model for tabular datasets and 20 configurations per model for imaging datasets. The full procedure was repeated 20 times for tabular datasets and 5 times for imaging datasets. Due to the high computational burden of adaptive ODE integration under this evaluation protocol, SODEN was evaluated on the imaging datasets but omitted from the tabular experiments.

Table~\ref{tab:realdata} shows that QSurv is consistently competitive across datasets, horizons, and evaluation metrics. Careful hyperparameter optimization makes several baselines strong, including CoxCC, CoxTime, NnetSurv, DeepHit, MDN, and DeSurv. Nevertheless, QSurv frequently achieves the best or second-best performance, particularly for discrimination and survival probability estimation at restricted horizons. These results suggest that directly parameterizing the instantaneous hazard, while evaluating the cumulative hazard through quadrature, provides a robust balance between flexible time-to-event modeling and stable likelihood-based training.

The advantage of QSurv is especially clear in the high-dimensional imaging settings, where flexible time-varying hazard modeling is most relevant. On COVID-19-NY, QSurv achieves the best full-horizon $C_{td}$ and the best median-horizon $C_{td}$, while remaining competitive with DeSurv and SODEN on likelihood- and calibration-oriented metrics. On C4KC-KiTS, QSurv obtains the best $C_{td}$ across all horizons and the strongest short-horizon IBS and IBLL. On BraTS, QSurv also achieves the best $C_{td}$ across all horizons and strong survival probability estimation, including the best median-horizon IBS and IBLL. These findings support the main motivation of QSurv: preserving continuous event times and directly learning instantaneous hazard can be beneficial when risk evolves nonlinearly over follow-up and covariates are high-dimensional.

On tabular datasets, performance differences are generally smaller, and several classical or discrete-time baselines remain highly competitive. QSurv nevertheless remains stable across datasets and horizons, often matching or improving on the best-performing alternatives. This suggests that the proposed quadrature-based hazard formulation does not sacrifice robustness in lower-dimensional settings, while providing additional flexibility for settings where proportional-hazards, parametric, or fixed-grid assumptions may be restrictive.

To further examine model behavior beyond aggregate prediction metrics, we compared predicted instantaneous hazard functions across representative risk clusters on COVID-19-NY (Figure~\ref{fig:covid}). Hazard clusters were formed by applying three-cluster $K$-means to subject-specific predicted hazard trajectories evaluated over a common time grid. This analysis highlights qualitative differences in how each method characterizes time-resolved risk. NnetSurv produces highly irregular, spike-like hazards. MDN yields smooth but relatively constrained, mostly monotone trajectories. SODEN produces nearly flat hazards that separate clusters mainly by risk level. DeSurv captures temporal variation, but its hazards remain irregular, with sharp oscillations and isolated spikes over follow-up. In contrast, QSurv produces smoother and more interpretable cluster-specific patterns: a high-risk cluster with elevated early hazard followed by gradual decline, an intermediate-risk cluster with persistent moderate risk, and a low-risk cluster with near-flat hazard. Unlike survival curves, which summarize accumulated risk up to each time point, these hazard trajectories directly identify when risk is concentrated over follow-up. The early high-risk pattern estimated by QSurv aligns with COVID-19 studies reporting concentrated early excess mortality after infection and elevated readmission-or-death risk shortly after hospital discharge, with attenuation over time \citep{iwashyna2023late,donnelly2021readmission}.

This comparison suggests that the qualitative difference between QSurv and SODEN may arise partly from how the cumulative hazard is evaluated during training, rather than from the hazard parameterization alone. Compared with SODEN, the fixed Gauss-Legendre quadrature rule evaluates the hazard at predetermined points over each subject's observed interval, rather than relying on adaptive solver behavior. Compared with DeSurv, QSurv treats the hazard itself as the primitive modeled quantity rather than recovering it indirectly from a learned distribution function. The resulting hazard estimates are flexible enough to capture non-monotone and cluster-specific temporal risk patterns, while remaining substantially smoother and more stable than the discrete-time or indirectly derived alternatives.

\begin{figure}[t]
    \centering
    \includegraphics[width=1.0\linewidth]{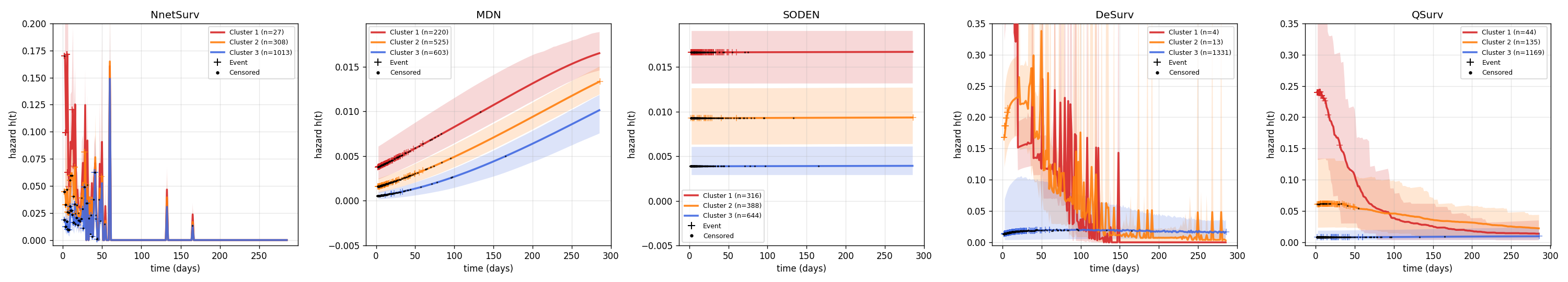}
    \vspace{-1.5em}
        \caption{Predicted instantaneous hazard functions on COVID-19-NY across representative risk clusters. Lines show cluster-level predicted hazards, with event times marked by plus signs and censored observations by dots. NnetSurv produces irregular spike-like hazards, MDN produces smooth but constrained trajectories, SODEN produces near-flat covariate-stratified hazards, and DeSurv captures temporal variation but with pronounced oscillations and spikes. QSurv yields smoother and more interpretable cluster-specific temporal patterns, including elevated early risk in the high-risk cluster followed by gradual decline over time.}
    \label{fig:covid} \vspace{-1em}
\end{figure}

\section{Conclusion}
We introduced QSurv, a quadrature-based nonparametric continuous-time survival model that directly parameterizes the instantaneous hazard as a flexible function of time and covariates. By evaluating the cumulative hazard with Gauss-Legendre quadrature, QSurv enables likelihood-based training without proportional-hazards assumptions, parametric survival distributions, time discretization, or adaptive ODE solvers. The resulting objective is simple, differentiable, and compatible with minibatch SGD, while Time-LoRA extends the framework to high-dimensional backbones by efficiently conditioning hazard predictions on time. Across simulations, tabular benchmarks, and imaging datasets, QSurv achieves competitive and robust performance. More importantly, its learned hazards provide time-resolved risk characterization, revealing clinically interpretable temporal patterns that may be obscured by aggregate survival summaries. A remaining limitation is that QSurv, like other deep survival models, depends on the hazard being learnable from available covariates and may require further validation under severe censoring, distribution shift, or highly irregular hazard dynamics. These results support QSurv as a practical framework for scalable and flexible continuous-time hazard modeling.

\begin{ack}
The research of Chaeyeon Lee and Sehwan Kim was supported by the Global-Learning \& Academic Research Institution for Master’s and PhD Students, and Postdocs (G-LAMP) Program of the National Research Foundation of Korea (NRF), funded by the Ministry of Education (No. RS-2025-25442252).
\end{ack}


\bibliographystyle{abbrvnat}
\bibliography{ref}

\newpage
\appendix

\section{Gauss-Legendre Quadrature}

Gauss-Legendre quadrature approximates a definite integral by a weighted sum of function evaluations at carefully chosen nonuniform nodes. Unlike grid-based rules such as Riemann sums or the trapezoidal rule, which evaluate the integrand at equally spaced points, Gauss-Legendre quadrature chooses nodes and weights to maximize algebraic precision. For an integral over the standard interval $[-1,1]$, the $K$-point rule is
\begin{equation}
    \int_{-1}^{1} q(u) du \approx \sum_{k=1}^{K} w_k q(\xi_k),
\end{equation}
where $\xi_k$ are the roots of the $K$-th Legendre polynomial $P_K$ and
\begin{equation}
    w_k = \frac{2}{(1-\xi_k^2)[P_K'(\xi_k)]^2}.
\end{equation}
The nodes and weights are fixed by the quadrature rule, not learned from data. With $K$ function evaluations, the rule is exact for all polynomials of degree at most $2K-1$.

For an integral over a general interval $[a,b]$, the affine transformation
\begin{equation}
    s = \frac{b-a}{2}u + \frac{b+a}{2}
\end{equation}
maps $u \in [-1,1]$ to $s \in [a,b]$. Therefore,
\begin{equation}
    \int_a^b q(s) ds \approx \frac{b-a}{2} \sum_{k=1}^{K} w_k q\left(\frac{b-a}{2}\xi_k + \frac{b+a}{2}\right).
\end{equation}
When $q$ is $2K$ times continuously differentiable, the classical Gauss-Legendre error formula gives
\begin{equation}
    E_K(q) = \frac{(b-a)^{2K+1}(K!)^4}{(2K+1)[(2K)!]^3} q^{(2K)}(\zeta)
\end{equation}
for some $\zeta \in (a,b)$. Thus, the approximation error is controlled by the high-order smoothness of the integrand. For analytic functions, Gauss-Legendre quadrature can converge much faster than low-order grid-based rules, often requiring substantially fewer evaluations to reach a comparable accuracy.

In QSurv, the integrand is the neural network parameterized hazard $q(s)=\lambda_\theta(s | x)$ and the interval is subject-specific, $[0,o_i]$. Applying the above transformation with $a=0$ and $b=o_i$ gives
\begin{equation}
    \Lambda_\theta(o_i | x_i) = \int_0^{o_i} \lambda_\theta(s | x_i) ds \approx \frac{o_i}{2} \sum_{k=1}^{K} w_k \lambda_\theta(o_i \tau_k | x_i),
\end{equation}
where $\tau_k=(\xi_k+1)/2$. This is the approximation used in the QSurv likelihood: the hazard remains a continuous-time function, while the cumulative hazard is evaluated through a fixed number of differentiable neural network parameterized hazard evaluations.

\section{Proofs}\label{appendix:proofs}

\begin{proof}[Proof of Theorem \ref{thm:error_bound}]
    The cumulative hazard is defined as the integral $\Lambda_\theta(t|x) = \int_0^t \lambda_\theta(s|x)\,ds$. To apply the standard error bound for Gauss-Legendre quadrature, which is defined on the reference interval $[-1, 1]$, we perform an affine transformation of variables. Let $s(u) = \tfrac{t}{2}(u + 1)$. This maps $u \in [-1, 1]$ to $s \in [0, t]$. The differential is $ds = \tfrac{t}{2}\, du$. The standard error term for $K$-point Gauss-Legendre quadrature on $[-1, 1]$ for a function $g$ is given by $E_K(g) = \tfrac{2^{2K+1}(K!)^4}{(2K+1)[(2K)!]^3}\, g^{(2K)}(\xi)$ for some $\xi \in (-1, 1)$. Let $g(u) = \lambda_\theta(\tfrac{t}{2}(u + 1)|x)$. By the chain rule, the $2K$-th derivative scales as $g^{(2K)}(u) = (\tfrac{t}{2})^{2K} \lambda_\theta^{(2K)}(s(u)|x)$. Evaluating at the point $\xi \in (-1, 1)$ given by the quadrature error theorem and letting $\tau = s(\xi) = (t/2)(\xi + 1) \in (0, t)$, we obtain $g^{(2K)}(\xi) = (t/2)^{2K}\, \lambda_\theta^{(2K)}(\tau|x)$. The total error is the standard error scaled by the integration measure $\tfrac{t}{2}$:
    \begin{equation}
        \text{Error} = \frac{t}{2} \cdot \frac{2^{2K+1}(K!)^4}{(2K + 1)[(2K)!]^3} \cdot \left[\left(\frac{t}{2}\right)^{2K} \lambda_\theta^{(2K)}(\tau|x)\right].
    \end{equation}
    Grouping terms, the factors of $2$ cancel, leaving the coefficient $\tfrac{t^{2K+1} (K!)^4}{(2K+1)[(2K)!]^3}$. Bounding $|\lambda_\theta^{(2K)}(\tau|x)|$ by $\max_{\tau \in [0, t]} |\lambda_\theta^{(2K)}(\tau|x)|$ yields the inequality.
\end{proof}

\begin{proof}[Proof of Corollary \ref{cor:error_prop}]
The negative log-likelihood decomposes into a log-hazard term and a cumulative hazard term. Since the hazard function $\lambda_\theta(t|x)$ is computed exactly via the neural network forward pass at time $t$, the term $-\delta \log \lambda_\theta(t|x)$ is identical in both the true and approximate losses. Consequently, the difference depends solely on the integral approximation:$$| \mathcal{L} - \hat{\mathcal{L}}_K | = | (-\delta \log \lambda + \Lambda) - (-\delta \log \lambda + \hat{\Lambda}_K) | = | \Lambda - \hat{\Lambda}_K | = \varepsilon_K(t|x).$$
\end{proof}

\section{Number of Quadrature Nodes}\label{appendix:num_nodes}

\begin{figure}[!t]
    \centering
    \includegraphics[width=0.9\linewidth]{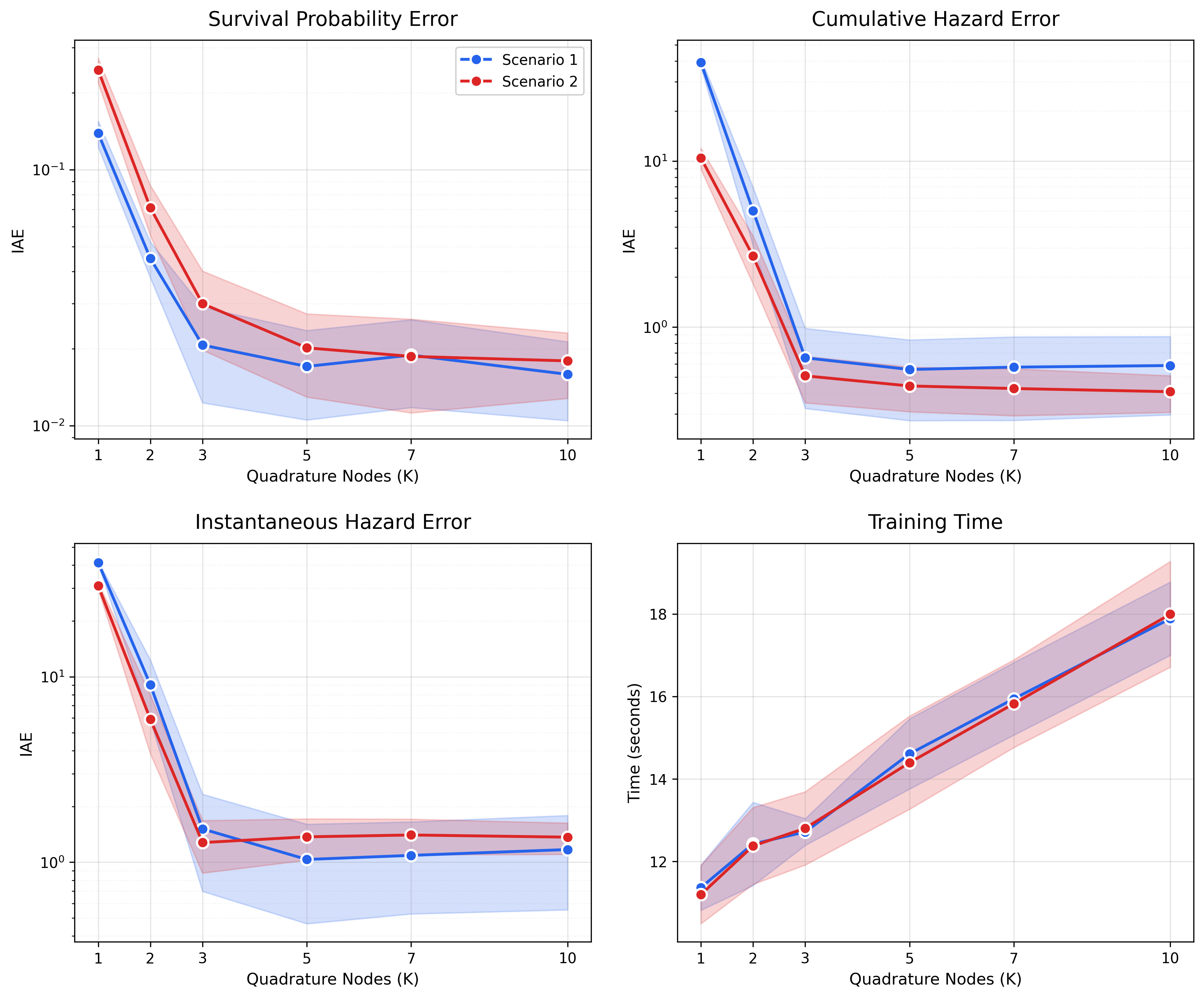}
    \caption{Convergence of approximation error and training efficiency on two simulation scenarios. The first three plots display the integrated absolute error (IAE) for survival probability, cumulative hazard, and instantaneous hazard as a function of quadrature nodes ($K$). All metrics exhibit rapid convergence, saturating beyond $K = 5$ across both simulation scenarios. The fourth panel confirms that while training time increases linearly with $K$. Shaded regions denote standard deviation ($n=20$). See Appendix \ref{appendix:simulation} and \ref{appendix:num_nodes} for simulation details and estimated curves, respectively.}
    \label{fig:num_nodes}
\end{figure}

\begin{figure}[!t]
    \centering
    \includegraphics[width=0.99\linewidth]{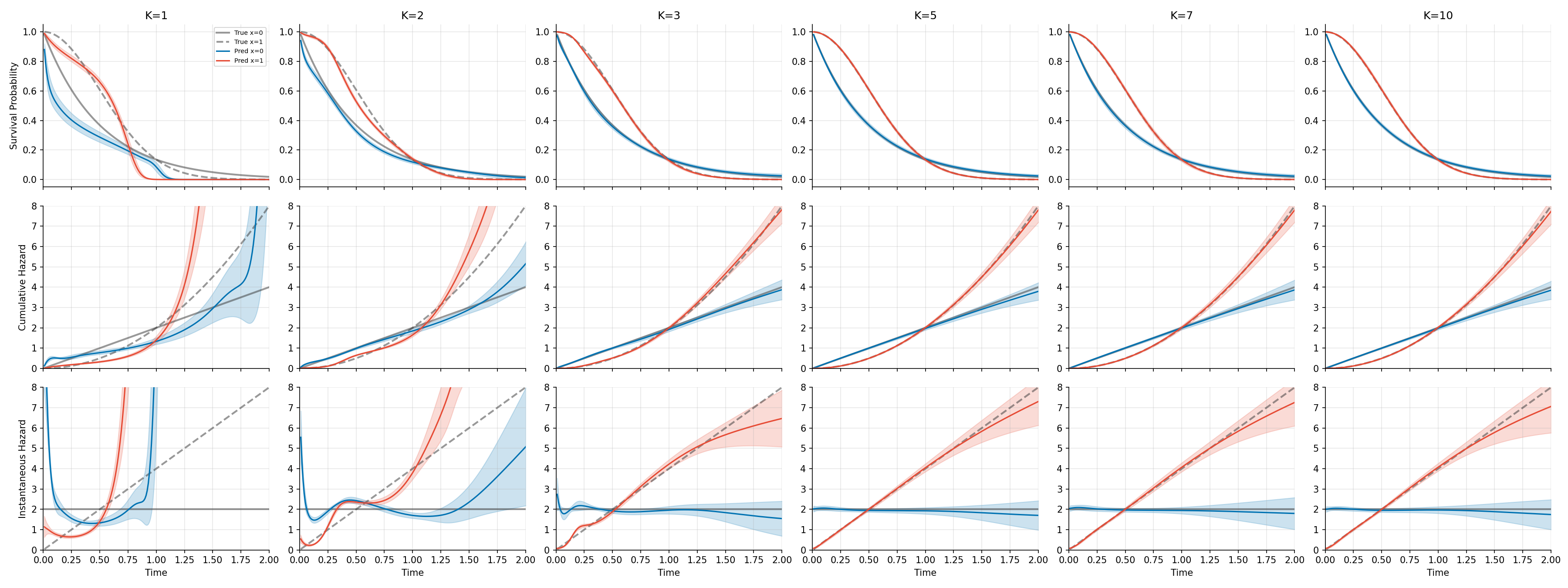}
    \caption{Instantaneous hazard, cumulative hazard, and survival functions for simulation scenario 1.}
    \label{fig:num_node_toy1}
\end{figure}

\begin{figure}[!t]
    \centering
    \includegraphics[width=0.99\linewidth]{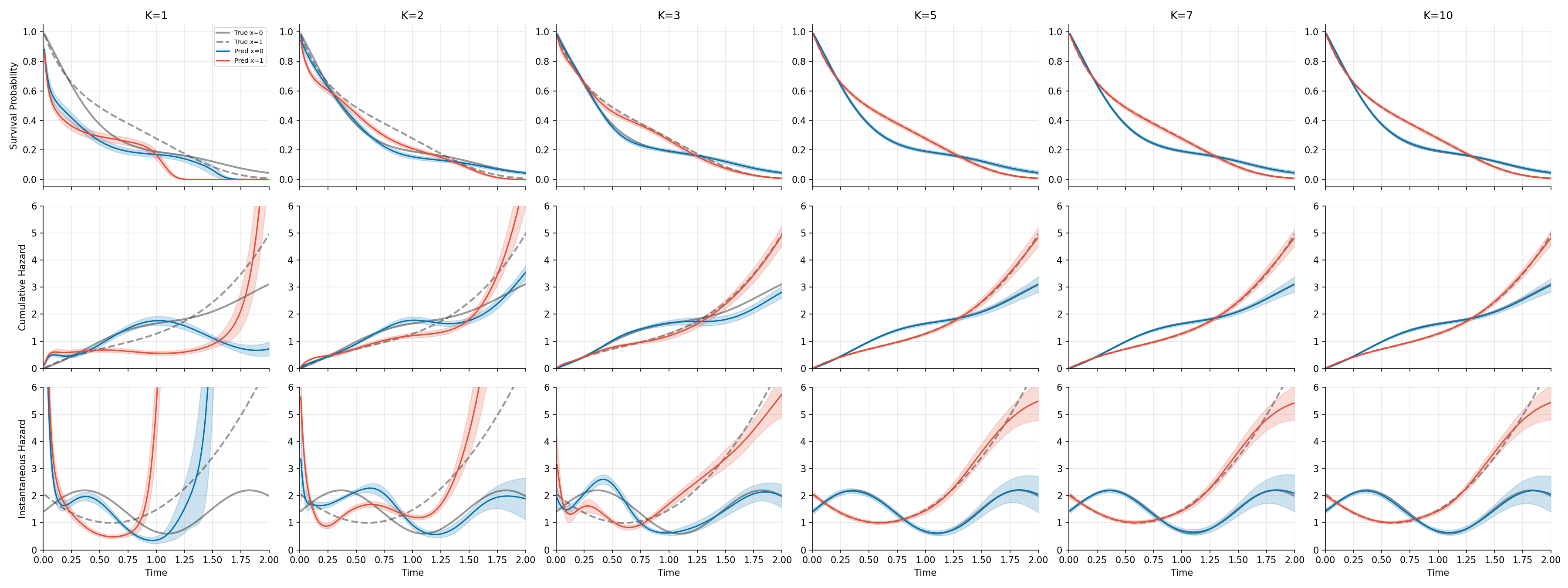}
    \caption{Instantaneous hazard, cumulative hazard, and survival functions for simulation scenario 2.}
    \label{fig:num_node_toy2}
\end{figure}

We further investigate how the number of quadrature nodes $K$ affects the estimation of the instantaneous hazard, cumulative hazard, and survival function. Although QSurv represents $\lambda_\theta(t | x)$ as a continuous-time neural function, the training objective evaluates this function only at finitely many time points: the observed event times through the log-hazard contribution and the subject-specific quadrature nodes through the cumulative-hazard contribution. Therefore, if $K$ is too small relative to the temporal complexity of the hazard, the quadrature-based objective may fail to resolve rapid changes in risk over time.

This observation clarifies the role of $K$ in QSurv. Unlike the number of bins in discrete-time survival models, $K$ does not define the temporal resolution of the hazard model itself. The hazard remains a continuous function of time. Instead, $K$ controls the numerical precision with which the cumulative hazard is approximated inside the likelihood. Reducing $K$ should therefore not be interpreted as imposing a simpler or more regularized hazard model. Rather, very small values of $K$ may lead to a poor approximation of the continuous-time likelihood, especially when the hazard has sharp or non-monotone temporal structure. Conversely, once $K$ is large enough to resolve the relevant temporal variation, increasing $K$ further should have little effect on estimation accuracy, aside from increased computation.

To empirically assess this behavior, we conducted a simulation study using two synthetic scenarios with complex, non-monotone hazards. Scenario 1 is a crossing-hazards setting with a binary covariate $x \sim \mathrm{Bernoulli}(0.5)$, defined by
\begin{equation}
    S(t | x)=\exp(-t^{1+x}), \qquad \lambda(t | x)=(1+x)t^x.
\end{equation}
Thus, the $x=0$ group follows $S(t | 0)=\exp(-t)$ with constant hazard $\lambda(t | 0)=1$, whereas the $x=1$ group follows $S(t | 1)=\exp(-t^2)$ with linearly increasing hazard $\lambda(t | 1)=2t$. The hazards intersect at $t=0.5$, while the survival curves cross at $t=1$, where $S(t | 0)=S(t | 1)=e^{-1}$. Censoring times are sampled from $C \sim \mathrm{Uniform}(0,2)$.

Scenario 2 introduces higher-frequency temporal variation through anti-phase sinusoidal hazards:
\begin{equation}
\begin{aligned}
    \lambda(t | x)&=1+0.8(1-2x)\sin(4t), \\
    \Lambda(t | x)&=t+0.2(1-2x)\{1-\cos(4t)\}, \\
    S(t | x)&=\exp(-\Lambda(t | x)).
\end{aligned}
\end{equation}
For $x=0$, the hazard is $1+0.8\sin(4t)$, while for $x=1$, the hazard is $1-0.8\sin(4t)$. This scenario creates oscillating risk dynamics with multiple hazard intersections, requiring the model to resolve rapid temporal changes in risk. Survival times are generated by inverse transform sampling, with censoring drawn from an exponential distribution with rate approximately $0.33$.

We trained QSurv models with $K \in \{1,2,3,5,7,10\}$ quadrature nodes across 20 independent runs. The results are summarized in Figure~\ref{fig:num_nodes}. We observe a rapid decrease in integrated absolute error for the instantaneous hazard, cumulative hazard, and survival function as $K$ increases. The error curves stabilize after a moderate number of quadrature nodes, indicating that the quadrature approximation has become sufficiently accurate for these settings. As expected, training time increases approximately linearly with $K$, reflecting the additional neural network parameterized hazard evaluations required at the quadrature nodes.

Figures~\ref{fig:num_node_toy1} and~\ref{fig:num_node_toy2} show the corresponding estimated trajectories. With very small values of $K$, the fitted hazard can exhibit visible approximation artifacts and may fail to capture important temporal features of the true hazard. In contrast, moderate values of $K$ recover stable estimates of the instantaneous hazard, cumulative hazard, and survival function. These findings support our practical use of a fixed moderate quadrature order in the main experiments.


\section{Simulation Study}\label{appendix:simulation}

\begin{table}[b!]
\centering
\caption{Data generation mechanisms. All covariate-dependent parameters are modeled as $\theta(x) = \exp(\mathcal{P}(x; \mathbf{w}))$ to ensure positivity, except for $\mu(x)$ in the Log-Normal distribution, which is modeled directly as $\mu(x) = \mathcal{P}(x; \mathbf{w})$. The polynomial function is defined as $\mathcal{P}(x; \mathbf{w}) = w_0 + w_1 x + w_2 x^2 + w_3 x^3$.}
\label{tab:simulation_equations_values}
\small
\begin{tabular}{lll}
\toprule
\textbf{Distribution} & \textbf{Hazard / PDF Definition} & \textbf{Parameter Coefficients} \\ 
\midrule
Exponential & $h(t|x) = \lambda(x)$ & $\mathbf{w}_\lambda = [-1.0, 0.5, -0.3, 0.15]$ \\
\addlinespace
Weibull & $h(t|x) = \frac{k(x)}{\lambda(x)} \left(\frac{t}{\lambda(x)}\right)^{k(x)-1}$ & $\mathbf{w}_k = [0.3, 0.2, -0.1, 0.05]$ \\
 & & $\mathbf{w}_\lambda = [2.0, 0.3, -0.2, 0.1]$ \\
\addlinespace
Gamma & $f(t|x) = \frac{\beta(x)^{k(x)}}{\Gamma(k(x))} t^{k(x)-1} e^{-\beta(x)t}$ & $\mathbf{w}_k = [1.8, 0.3, -0.1, 0.05]$ \\
 & & $\mathbf{w}_\beta = [0.3, -0.4, 0.15, -0.05]$ \\
\addlinespace
Gompertz & $h(t|x) = b(x) e^{c t}$ & $\mathbf{w}_b = [-2.0, 0.4, -0.2, 0.1]$ \\
 & & $c = 0.05$  \\
\addlinespace
Log-Normal & $T = \exp(\mu(x) + \sigma(x) Z), \; Z \sim \mathcal{N}(0,1)$ & $\mathbf{w}_\mu = [1.5, 0.8, -0.4, 0.2]$ \\
 & & $\mathbf{w}_\sigma = [-0.1, 0.25, -0.10, 0.03]$ \\
\addlinespace
Log-Logistic & $h(t|x) = \frac{(\beta(x)/\alpha(x)) (t/\alpha(x))^{\beta(x)-1}}{1 + (t/\alpha(x))^{\beta(x)}}$ & $\mathbf{w}_\alpha = [1.2, 0.4, -0.15, 0.08]$ \\
 & & $\mathbf{w}_\beta = [1.0, 0.3, -0.1, 0.05]$ \\
\bottomrule
\end{tabular}%
\end{table}

To evaluate the capability of the proposed models to capture complex, non-linear relationships between covariates and survival outcomes, we generated synthetic datasets using six parametric distributions: Exponential, Weibull, Gamma, Gompertz, Log-Normal, and Log-Logistic. For all scenarios, we simulated a one-dimensional covariate $x$ drawn from a uniform distribution between -1 and 1. The distribution parameters, such as shape, scale, or rate, were modeled as non-linear functions of $x$. We employed third-order polynomial transformations followed by an exponential link function to ensure positivity, except for the location parameter in the Log-Normal distribution. The polynomial coefficients were manually tuned to ensure diverse hazard shapes and stable event times within the observation window.

The chosen distributions cover a wide range of survival behaviors. The Exponential and Gompertz scenarios represent Proportional Hazards settings, whereas the Weibull, Gamma, Log-Normal, and Log-Logistic distributions induce non-proportional hazards where the shape of the distribution changes conditional on the covariate. Notably, the Log-Normal and Log-Logistic distributions generate unimodal hazard functions that rise and then fall, mimicking patterns often seen in medical contexts.

Censoring times were generated independently of the covariates using a uniform distribution, where the upper bound was dynamically calibrated for each dataset to achieve a target censoring rate of approximately 20\%. The observed time was defined as the minimum of the true event time and the censoring time. For each simulation run, we generated 2,000 training and testing samples, respectively. The experiments were repeated 20 times with different random seeds to report the mean and standard deviation of the results. All neural network baselines used a standardized architecture with two hidden layers of 32 units and Tanh activations. To strictly assess how well the models recovered the underlying distributions, we calculated the L1 error for the hazard, cumulative hazard, and survival functions on the testing samples by comparing the model estimates against the true analytical functions.

To assess the model's ability to recover the global system dynamics independent of individual covariate variations, we computed the marginalized survival, cumulative hazard, and instantaneous hazard curves. These are obtained by averaging the conditional functions over the distribution of the covariate space $\mathcal{X}$:
\begin{equation}
\begin{aligned}
    \bar{S}(t) & = \int_{\mathcal{X}} S(t|x) p(x) dx \approx \frac{1}{n} \sum_{i=1}^{n} \hat{S}(t|x_i), \\
    \bar{\Lambda}(t) & = \int_{\mathcal{X}} \Lambda(t|x) p(x) dx \approx \frac{1}{n} \sum_{i=1}^{n} \hat{\Lambda}(t|x_i), \\
    \bar{\lambda}(t) & = \int_{\mathcal{X}} \lambda(t|x) p(x) dx \approx \frac{1}{n} \sum_{i=1}^{n} \hat{\lambda}(t|x_i).
\end{aligned}
\end{equation}
By comparing these marginalized trajectories against the true analytical integrations, we can visualize whether the model correctly captures the aggregate behavior of the population, including the mean hazard shape and the overall survival decay, even when the underlying individual parameters $\theta$ vary non-linearly. Figures \ref{fig:sim_exponential}--\ref{fig:sim_loglogistic} present the results.

\begin{table}[!t]
\centering
\caption{Performance comparison of survival models across various distributions. Values are Mean $\pm$ Std (20 seeds). Best results are \textbf{bold}, second best are \underline{underlined}. Training times are wall-clock seconds on a single Apple~M4 CPU (10-core, 32~GB RAM); relative comparisons across rows are meaningful but absolute values are machine-dependent.}
\label{tab:simulation_full}
\resizebox{\textwidth}{!}{%
\begin{tabular}{ccccccccc}
\toprule
\textbf{Metric} & \textbf{Distribution} & \textbf{CoxCC} & \textbf{CoxTime} & \textbf{NnetSurv} & \textbf{MDN} & \textbf{DeSurv} & \textbf{SODEN} & \textbf{QSurv} \\
\midrule
\multirow{6}{*}{\shortstack[c]{\textbf{$L_1$ Error of} \\ \textbf{$\hat{S}$}}}
 & Exponential  & \textbf{0.0122 $\pm$ 0.0038} & \underline{0.0130 $\pm$ 0.0042} & 0.0266 $\pm$ 0.0084 & 0.0155 $\pm$ 0.0055 & 0.0177 $\pm$ 0.0057 & 0.0167 $\pm$ 0.0069 & 0.0162 $\pm$ 0.0058 \\
 & Weibull      & 0.0268 $\pm$ 0.0041 & 0.0144 $\pm$ 0.0038 & 0.0236 $\pm$ 0.0041 & 0.0354 $\pm$ 0.0037 & \textbf{0.0127 $\pm$ 0.0043} & 0.0139 $\pm$ 0.0053 & \underline{0.0133 $\pm$ 0.0046} \\
 & Gamma        & 0.0197 $\pm$ 0.0041 & 0.0138 $\pm$ 0.0035 & 0.0213 $\pm$ 0.0032 & 0.1056 $\pm$ 0.0028 & 0.0121 $\pm$ 0.0035 & \textbf{0.0107 $\pm$ 0.0029} & \underline{0.0110 $\pm$ 0.0030} \\
 & Gompertz     & 0.0121 $\pm$ 0.0040 & 0.0143 $\pm$ 0.0054 & 0.0276 $\pm$ 0.0058 & 0.0202 $\pm$ 0.0043 & 0.0129 $\pm$ 0.0043 & \textbf{0.0115 $\pm$ 0.0034} & \underline{0.0120 $\pm$ 0.0032} \\
 & Log-Normal    & 0.0166 $\pm$ 0.0065 & 0.0166 $\pm$ 0.0049 & 0.0199 $\pm$ 0.0054 & 0.0272 $\pm$ 0.0034 & \textbf{0.0152 $\pm$ 0.0054} & 0.0165 $\pm$ 0.0054 & \underline{0.0160 $\pm$ 0.0064} \\
 & Log-Logistic  & 0.0356 $\pm$ 0.0040 & 0.0157 $\pm$ 0.0042 & 0.0209 $\pm$ 0.0035 & 0.0567 $\pm$ 0.0041 & \textbf{0.0144 $\pm$ 0.0041} & \underline{0.0150 $\pm$ 0.0042} & 0.0154 $\pm$ 0.0042 \\
\midrule
\multirow{6}{*}{\shortstack[c]{\textbf{$L_1$ Error of} \\ \textbf{$\hat{\Lambda}$}}}
 & Exponential  & \textbf{0.1780 $\pm$ 0.0479} & \underline{0.2013 $\pm$ 0.0638} & 0.4163 $\pm$ 0.1258 & 0.2244 $\pm$ 0.0958 & 0.2187 $\pm$ 0.0900 & 0.2144 $\pm$ 0.0827 & 0.2221 $\pm$ 0.0923 \\
 & Weibull      & 0.7486 $\pm$ 0.1803 & 0.2372 $\pm$ 0.1229 & 0.5065 $\pm$ 0.1114 & 0.5137 $\pm$ 0.1197 & 0.1974 $\pm$ 0.0644 & \textbf{0.1887 $\pm$ 0.0842} & \underline{0.1894 $\pm$ 0.0938} \\
 & Gamma        & 2.4663 $\pm$ 0.3089 & 1.1946 $\pm$ 0.6717 & 1.3610 $\pm$ 0.4712 & 2.8405 $\pm$ 0.4260 & 0.6507 $\pm$ 0.2066 & \underline{0.5796 $\pm$ 0.2705} & \textbf{0.5766 $\pm$ 0.2472} \\
 & Gompertz     & 0.2773 $\pm$ 0.1030 & 0.3699 $\pm$ 0.2559 & 0.7631 $\pm$ 0.1900 & 0.6015 $\pm$ 0.1340 & 0.2870 $\pm$ 0.1289 & \textbf{0.1964 $\pm$ 0.1009} & \underline{0.1990 $\pm$ 0.1021} \\
 & Log-Normal    & 0.8253 $\pm$ 0.2925 & 0.6984 $\pm$ 0.3496 & \underline{0.4052 $\pm$ 0.1085} & 1.0280 $\pm$ 0.4111 & \textbf{0.3427 $\pm$ 0.1220} & 0.4213 $\pm$ 0.1943 & 0.4205 $\pm$ 0.2091 \\
 & Log-Logistic  & 0.6392 $\pm$ 0.1234 & 0.2488 $\pm$ 0.0619 & 0.3262 $\pm$ 0.0438 & 0.4433 $\pm$ 0.0416 & \textbf{0.2279 $\pm$ 0.1152} & 0.2358 $\pm$ 0.1280 & \underline{0.2305 $\pm$ 0.1158} \\
\midrule
\multirow{6}{*}{\shortstack[c]{\textbf{$L_1$ Error of} \\ \textbf{$\hat{\lambda}$}}}
 & Exponential  & 0.1775 $\pm$ 0.0410 & 0.1775 $\pm$ 0.0422 & 0.3830 $\pm$ 0.0438 & 0.0568 $\pm$ 0.0446 & 0.0448 $\pm$ 0.0159 & \textbf{0.0330 $\pm$ 0.0138} & \underline{0.0352 $\pm$ 0.0164} \\
 & Weibull      & 0.1576 $\pm$ 0.0322 & 0.1182 $\pm$ 0.0207 & 0.2383 $\pm$ 0.0254 & 0.0644 $\pm$ 0.0102 & 0.0295 $\pm$ 0.0090 & \textbf{0.0272 $\pm$ 0.0119} & \underline{0.0282 $\pm$ 0.0139} \\
 & Gamma        & 0.9012 $\pm$ 0.0921 & 0.5661 $\pm$ 0.1769 & 0.7249 $\pm$ 0.0370 & 0.5646 $\pm$ 0.0459 & 0.2930 $\pm$ 0.0907 & \textbf{0.1716 $\pm$ 0.0901} & \underline{0.1718 $\pm$ 0.0859} \\
 & Gompertz     & 0.1713 $\pm$ 0.0262 & 0.1890 $\pm$ 0.0549 & 0.2921 $\pm$ 0.0276 & 0.0920 $\pm$ 0.0140 & 0.0502 $\pm$ 0.0278 & \underline{0.0262 $\pm$ 0.0174} & \textbf{0.0258 $\pm$ 0.0180} \\
 & Log-Normal    & 0.1235 $\pm$ 0.0248 & 0.1114 $\pm$ 0.0294 & 0.1938 $\pm$ 0.0219 & 0.1882 $\pm$ 0.0520 & 0.0473 $\pm$ 0.0201 & \underline{0.0462 $\pm$ 0.0183} & \textbf{0.0457 $\pm$ 0.0185} \\
 & Log-Logistic  & 0.1621 $\pm$ 0.0210 & 0.1266 $\pm$ 0.0242 & 0.3290 $\pm$ 0.0525 & 0.1031 $\pm$ 0.0090 & \underline{0.0513 $\pm$ 0.0299} & 0.0517 $\pm$ 0.0295 & \textbf{0.0491 $\pm$ 0.0291} \\
\midrule
\multirow{6}{*}{\shortstack[c]{\textbf{Training} \\ \textbf{Time (s)}}}
 & Exponential  & 5.3210 $\pm$ 0.7844 & 6.7244 $\pm$ 0.8221 & 2.4345 $\pm$ 0.2233 & 3.5458 $\pm$ 0.2663 & 6.1273 $\pm$ 1.2448 & 87.9615 $\pm$ 8.5607 & 6.0749 $\pm$ 0.8276 \\
 & Weibull      & 5.0815 $\pm$ 0.4088 & 6.6068 $\pm$ 0.5126 & 2.4946 $\pm$ 0.4445 & 3.6100 $\pm$ 0.5023 & 5.9509 $\pm$ 0.6691 & 124.1380 $\pm$ 17.2815 & 5.7566 $\pm$ 0.4297 \\
 & Gamma        & 4.6292 $\pm$ 0.3867 & 6.3100 $\pm$ 0.6772 & 2.2909 $\pm$ 0.2931 & 3.4942 $\pm$ 0.4543 & 6.0968 $\pm$ 0.9452 & 149.9774 $\pm$ 24.4947 & 6.2316 $\pm$ 1.1584 \\
 & Gompertz     & 4.4591 $\pm$ 0.2230 & 5.9313 $\pm$ 0.3321 & 1.9610 $\pm$ 0.1278 & 3.0111 $\pm$ 0.0375 & 5.5127 $\pm$ 0.4924 & 82.9457 $\pm$ 10.5067 & 5.3545 $\pm$ 0.3065 \\
 & Log-Normal    & 5.8264 $\pm$ 2.2276 & 7.8568 $\pm$ 3.4891 & 2.8560 $\pm$ 1.1757 & 4.3230 $\pm$ 1.8299 & 8.7879 $\pm$ 6.4600 & 243.5298 $\pm$ 117.4764 & 8.0038 $\pm$ 4.6329 \\
 & Log-Logistic  & 4.3939 $\pm$ 0.8087 & 5.8667 $\pm$ 1.2868 & 2.3401 $\pm$ 1.8595 & 3.3814 $\pm$ 0.9135 & 5.0744 $\pm$ 0.6543 & 131.3422 $\pm$ 13.7360 & 5.7363 $\pm$ 0.6092 \\
\bottomrule
\end{tabular}%
}
\end{table}

\section{LoRA Ablation Study}
\label{appendix:lora_ablation}

\begin{table}[!t]
    \small
    \centering
    \setlength{\tabcolsep}{4pt}
    \caption{Comparison of QSurv conditioning variants on $L_1$ error of instantaneous hazard (smaller the better) across different parametric distributions, with standard deviations over 20 random seeds. Best results are \textbf{bold}, second best are \underline{underlined}.}
    \label{tab:simulation_qsurv_variants}
    \begin{tabular}{cccc}
    \toprule
    \textbf{Distribution} & \textbf{QSurv (Concat)} & \textbf{QSurv (FiLM)} & \textbf{QSurv (LoRA)} \\
    \midrule
    Exponential  & \underline{0.0327 $\pm$ 0.0159} & \textbf{0.0260 $\pm$ 0.0104} & 0.0352 $\pm$ 0.0164 \\
    Weibull      & \textbf{0.0278 $\pm$ 0.0055} & 0.0345 $\pm$ 0.0182 & \underline{0.0282 $\pm$ 0.0139} \\
    Gamma        & \textbf{0.1199 $\pm$ 0.0344} & 0.2033 $\pm$ 0.0894 & \underline{0.1718 $\pm$ 0.0859} \\
    Gompertz     & 0.0506 $\pm$ 0.0157 & \underline{0.0421 $\pm$ 0.0403} & \textbf{0.0258 $\pm$ 0.0180} \\
    Log-normal   & \textbf{0.0431 $\pm$ 0.0154} & 0.0469 $\pm$ 0.0197 & \underline{0.0457 $\pm$ 0.0185} \\
    Log-logistic & 0.0939 $\pm$ 0.0283 & \underline{0.0590 $\pm$ 0.0507} & \textbf{0.0491 $\pm$ 0.0291} \\
    \bottomrule
    \end{tabular}
\end{table}

We performed an ablation study to compare different time-conditioning mechanisms within QSurv. Specifically, we compared three variants: direct concatenation of time and covariates, FiLM-based conditioning, and the proposed Time-LoRA conditioning. All variants used the same simulation settings, training protocol, and evaluation metric. Performance was measured by the \(L_1\) error between the predicted and true instantaneous hazard functions, averaged over 20 random seeds.

Table~\ref{tab:simulation_qsurv_variants} shows that no single conditioning mechanism uniformly dominates across all data-generating distributions. Direct concatenation performs best for Weibull, Gamma, and log-normal hazards, suggesting that simple conditioning can be sufficient when the hazard structure is relatively easy to represent in the low-dimensional simulation setting. FiLM performs best for the exponential setting and remains competitive for Gompertz and log-logistic hazards. Time-LoRA achieves the best performance for Gompertz and log-logistic hazards, and is second best for Weibull, Gamma, and log-normal distributions.

These results suggest that Time-LoRA is a competitive time-conditioning mechanism for hazard modeling, particularly in settings where temporal risk patterns require more flexible interactions between time and covariates. At the same time, the ablation indicates that in simple low-dimensional simulations, direct concatenation can remain a strong baseline. The main motivation for Time-LoRA is therefore not that it uniformly improves all synthetic settings, but that it provides an efficient and flexible mechanism for conditioning high-dimensional neural backbones on time while preserving competitive hazard-estimation accuracy.

Figures \ref{fig:sim_exponential}--\ref{fig:sim_loglogistic} present the visualizations of the fitted curves.

\section{Experiment Details}\label{appendix:experiments}

\subsection{Datasets}\label{appendix:datasets}

We evaluated all methods on a mixture of tabular and imaging survival datasets. The main benchmark includes six tabular datasets (GBSG, METABRIC, NWTCO, Framingham, FLCHAIN, and SUPPORT2) and three imaging datasets (COVID-19-NY, C4KC-KiTS, and BraTS).

\textbf{GBSG} comes from a multicenter randomized clinical trial conducted by the German Breast Cancer Study Group to evaluate adjuvant treatment strategies in patients with lymph node-positive breast cancer \citep{schumacher1994randomized}. The endpoint is breast cancer recurrence after treatment. The public dataset contains 686 patients, with 299 observed recurrence events (43.6\%) and 387 right-censored observations (56.4\%). We used eight covariates covering demographic features, tumor characteristics, hormone receptor status, and treatment information. Categorical variables were one-hot encoded.

\textbf{METABRIC} comes from the Molecular Taxonomy of Breast Cancer International Consortium and was designed to study molecular subtypes of breast cancer \citep{curtis2012genomic}. We used the version distributed with Pycox, following the Immunohistochemical 4 plus Clinical (IHC4+C) covariate set \citep{katzman2018deepsurv}. The input variables include four gene-expression markers (MKI67, EGFR, PGR, and ERBB2) and five clinical variables (hormone therapy, radiotherapy, chemotherapy, estrogen receptor status, and age). The dataset contains 1,904 patients, with 1,103 deaths (57.9\%) and 801 censored observations (42.1\%). Missing categorical values were imputed with the most frequent category before one-hot encoding.

\textbf{NWTCO} is from the third and fourth National Wilms Tumor Study Group trials and includes pediatric patients diagnosed with Wilms tumor between 1979 and 1994 \citep{breslow1999design}. We analyzed 4,028 patients with complete relapse outcome and covariate information. The covariates include disease stage, age, and subcohort membership. There were 571 observed events (14.1\%) and 3,457 censored observations (85.8\%).

\textbf{Framingham} uses the Framingham Heart Study longitudinal teaching dataset distributed by NHLBI/BioLINCC \citep{nhlbi2022biolincc}. This is a Framingham-derived teaching dataset rather than the full Framingham cohort. It contains longitudinal clinic, questionnaire, laboratory, and adjudicated event data on 4,434 participants from three examination periods, yielding 11,627 person-exam observations and 39 variables. Outcomes include angina, myocardial infarction, stroke or cardiovascular disease, hypertension, and death, with corresponding time-to-event or censoring variables. In our benchmark, we used cardiovascular disease (CVD) as the survival endpoint, with \texttt{CVD} as the event indicator and \texttt{TIMECVD} as the observed follow-up time. CVD events were observed in 996 observations (23.3\%), and 3,277 observations (76.7\%) were right-censored. 

\textbf{FLCHAIN} comes from a study of serum free light chains and mortality among residents aged 50 years or older in Olmsted County, Minnesota \citep{dispenzieri2012use}. The final sample contains 7,874 participants selected by stratified random sampling over age and sex. We used six covariates: age, creatinine, kappa and lambda free light chain levels, MGUS status, and sex. There were 2,167 deaths (27.5\%) and 5,707 censored observations (72.5\%). Missing continuous values were imputed with the mean.

\textbf{SUPPORT2} comes from the SUPPORT Phase II study, a multicenter study of seriously ill hospitalized adults conducted between 1992 and 1994 \citep{connors1995controlled}. We used 9,105 observations with 47 covariates, including demographics, disease severity, physiologic measurements, and clinical status variables. Death was observed in 6,201 observations (68.1\%), and 2,904 observations (31.9\%) were right-censored. Observations with follow-up of three days or less were excluded so that baseline covariates preceded the event time.

\textbf{COVID-19-NY} consists of patients who tested positive for COVID-19 at Stony Brook University. We used the linked clinical file and AP chest radiographs from the TCIA release. For each patient, the radiograph study closest to the hospital visit date was selected among portable/AP chest views, and the first available DICOM image from that study was converted to a grayscale PNG. The endpoint was in-hospital death, coded from discharge status, and observed time was measured in days from the selected imaging study to discharge or death. Images were padded to a square field of view and resized to $224 \times 224$ before model input.

\textbf{C4KC-KiTS} is a TCIA renal cancer imaging dataset with arterial-phase abdominal CT and kidney/tumor segmentations. Original imaging consisted of CT DICOM series and corresponding segmentation objects. For the 2D slice selection, CT volumes were resampled to 1 mm isotropic spacing, clipped to an abdominal soft-tissue window ($[-160,240]$ HU), and reduced to the axial slice with the largest kidney tumor cross-section according to the segmentation mask. The selected slice was resized to $224 \times 224$ and used as a single-channel grayscale image. The endpoint was overall survival after surgery: observed time was \texttt{vital\_days\_after\_surgery}, and death was coded from \texttt{vital\_status}; alive or censored patients were treated as right-censored. Only malignant cases with valid follow-up were included.

\textbf{BraTS} uses the BraTS 2020 glioma MRI training data \citep{menze2015brats, bakas2017advancing, bakas2018identifying}. Original imaging consisted of multi-parametric brain MRI volumes with segmentation masks; in this benchmark we used the FLAIR sequence. For the 2D slice selection, we extracted the axial FLAIR slice with the largest tumor cross-section using the segmentation mask, falling back to the center slice only if the mask was unavailable or empty. Intensities were clipped at the within-slice 99th percentile, normalized, and saved as $224 \times 224$ grayscale images. The endpoint was overall survival time in days from the BraTS survival file. We constructed the event indicator from the available vital-status/alive annotation in the processed survival file; subjects annotated as alive with follow-up time were treated as right-censored, while subjects without an alive annotation were treated as observed events.

\subsection{Evaluation Horizons}

Metrics were reported at three horizons: the full supported time horizon and two shorter horizons. The full horizon was chosen from the training split by estimating the censoring survival function and taking the largest time point with adequate censoring support, operationalized in the code as $\hat{G}(\tau) \geq 0.001$. The IPCW weights themselves were then capped during metric computation, as described below. The Q1 and Q2 horizons were defined from the observed time distribution of the held-out test split. These horizons were used only for reporting, not for model fitting or hyperparameter selection. In datasets with administrative censoring pile-ups, some upper quantile horizons can coincide with the full horizon; this is reported as part of the dataset behavior rather than manually adjusted.

\subsection{Evaluation Metrics}\label{appendix:metrics}

\textbf{Time-dependent C-index.} We used an IPCW version of the time-dependent C-index to evaluate discrimination under right censoring \citep{uno2011c}. Let $\hat{S}_i(t)=\hat{S}_\theta(t|x_i)$ be the predicted survival probability for subject $i$, and let $\hat{G}(t)=P(C>t)$ be the censoring survival function estimated by Kaplan-Meier on the training split. Comparable pairs are formed from subjects with observed events before the evaluation horizon. A pair is concordant when the subject who fails earlier is assigned a lower predicted survival probability at that event time.

\[
C_\tau^{td} = 
\frac{\sum_{i, \delta_i=1, o_i<\tau} \sum_{j: o_i < o_j} \mathbb{I}\{\hat{S}_i(o_i)<\hat{S}_j(o_i)\} w_i}
     {\sum_{i, \delta_i=1, o_i<\tau} \sum_{j: o_i < o_j} w_i},
\qquad
w_i = \min\{1/\hat{G}^2(o_i), 10\}.
\]

The cap at 10 was used for all IPCW-based metrics. It is a numerical stabilization step, not a subject-level exclusion rule. Subjects were not removed simply because they occurred late in follow-up.

\textbf{Integrated Brier Score.} The Brier score measures the squared error between predicted survival and the observed survival status at a fixed time. We used the IPCW form proposed for right-censored survival data \citep{graf1999assessment}. At time $t$,

\[
\text{BS}(t) = \frac{1}{n} \sum_{i=1}^{n} \Biggl[
\frac{\bigl(\hat{S}_i(t)\bigr)^2 \mathbb{I}(o_i < t, \delta_i = 1)}{\hat{G}(o_i)}
+
\frac{\bigl(1 - \hat{S}_i(t)\bigr)^2 \mathbb{I}(o_i > t)}{\hat{G}(t)}
\Biggr],
\]

with the same IPCW cap applied to the inverse censoring weights. The integrated Brier score (IBS) was computed by numerical integration over the evaluation interval. Lower IBS indicates better calibrated survival prediction.

\textbf{Integrated binomial log-likelihood.} We also report an IPCW binomial log-likelihood, which is sensitive to both calibration and discrimination. At time $t$,

\[
\text{BLL}(t) = \frac{1}{n} \sum_{i=1}^{n} \Biggl[
\frac{\log\bigl(1 - \hat{S}_i(t)\bigr) \mathbb{I}(o_i < t, \delta_i = 1)}{\hat{G}(o_i)}
+
\frac{\log\bigl(\hat{S}_i(t)\bigr) \mathbb{I}(o_i > t)}{\hat{G}(t)}
\Biggr].
\]

The integrated version, denoted IBLL, was computed over the same time grid as IBS. Because this is a log-likelihood rather than a loss, higher values are better.

\textbf{D-calibration.}
We also report distributional calibration (D-calibration), following \citep{haider2020effective}. D-calibration assesses whether the model's predicted individual survival distributions are calibrated over the full follow-up range, rather than at a single fixed time point. Under a correctly specified continuous survival model, the probability integral transform implies that $F(T | X)$ is uniformly distributed on $[0,1]$; equivalently, $S(T | X)$ is also uniformly distributed on $[0,1]$. Thus, for uncensored subjects, we evaluate the predicted survival probability at the observed event time, $\hat{S}_i(o_i)$.

Right-censoring requires a partial contribution because the true event time is only known to exceed $o_i$. On the survival probability scale, the unobserved value $\hat{S}_i(T_i)$ must lie in the interval $[0,\hat{S}_i(o_i)]$. Following the standard D-calibration construction, each censored subject's contribution is therefore distributed uniformly over this compatible interval. We divide $[0,1]$ into 10 equal-width bins and compare the resulting bin counts with the uniform distribution using a chi-square goodness-of-fit statistic with 9 degrees of freedom. We report the corresponding $p$-value, where larger values indicate less evidence against distributional calibration. Because this is a hypothesis-test diagnostic rather than an effect-size metric, we use D-calibration as a secondary calibration assessment rather than as the primary model-ranking criterion.

\subsection{Implementation Details}

\subsubsection*{Tabular Data}

For tabular datasets, all models used a multilayer perceptron (MLP) backbone to extract fixed-dimensional representations from input covariates. The MLP consists of multiple hidden layers with nonlinear activations, dropout, and optional batch normalization. Within each experiment, all models used the same backbone search space, followed by model-specific output heads determined by each survival modeling formulation.

We performed hyperparameter optimization using random search. Each candidate configuration was trained for up to 200 epochs, and the final configuration was selected based on the best validation $C_{td}$, with negative validation IBS used as a tie-breaker. We used AdamW as the optimizer with a cosine learning-rate schedule, setting $T_{\max}$ to the maximum number of epochs and $\eta_{\min}=0$. For each tabular dataset and model, we evaluated 30 random hyperparameter configurations. The search space is summarized in Table~\ref{tab:hpo_tabular}. For NnetSurv and DeepHit, the number of discrete time bins was fixed at 50. For MDN, the number of mixture components was fixed at 5.

\begin{table}[h]
\centering
\caption{Hyperparameter search space for tabular datasets.}
\label{tab:hpo_tabular}
\begin{tabular}{l c}
\toprule
Hyperparameter & Search space \\
\midrule
Number of dense hidden layers & $\{2,3,4\}$ \\
Hidden size & $\{32,64,128,256\}$ \\
Learning rate & LogUniform$(10^{-4},10^{-2})$ \\
Weight decay & LogUniform$(10^{-8},10^{-3})$ \\
Dropout & $\{0.0,0.1,0.3,0.5\}$ \\
Batch size & $\{64,128,256\}$ \\
Batch normalization & $\{\texttt{True},\texttt{False}\}$ \\
\bottomrule
\end{tabular}
\end{table}

\subsubsection*{Medical Imaging Data}

For medical imaging datasets, all models used ResNet-18 as the convolutional backbone. The original fully connected layer was removed, and the extracted feature vector was passed to a model-specific survival head. The first convolutional layer was modified when the input channel dimension differed from the standard three-channel RGB setting. For single-channel inputs with pretrained weights, the first-layer filters were initialized by averaging the pretrained RGB filters across channels. As shown in Figure~\ref{fig:resnet18_backbone}, all image-based models shared the same ResNet-18 feature extractor, while the final prediction heads differed according to each model's output parameterization.

We performed random-search hyperparameter optimization for medical imaging models using the search space in Table~\ref{tab:hpo_image}. Each candidate configuration was trained for up to 200 epochs, and the final configuration was selected using validation $C_{td}$, with negative validation IBS used as a tie-breaker. We used AdamW with a cosine learning-rate schedule, again setting $T_{\max}$ to the maximum number of epochs and $\eta_{\min}=0$. For each imaging dataset and model, we evaluated 20 random hyperparameter configurations. Hidden-layer dimensions were not searched because the feature dimension was fixed by the ResNet-18 backbone. As in the tabular experiments, the number of time bins was fixed at 50 for NnetSurv and DeepHit, and the number of mixture components was fixed at 5 for MDN.

\begin{table}[h]
\centering
\caption{Hyperparameter search space for medical imaging datasets. Hidden-layer dimensions are not searched because the feature dimension is fixed by the convolutional backbone.}
\label{tab:hpo_image}
\begin{tabular}{l c}
\toprule
Hyperparameter & Search space \\
\midrule
Learning rate & LogUniform$(10^{-5},10^{-2})$ \\
Weight decay & LogUniform$(10^{-8},10^{-3})$ \\
Dropout & $\{0.0,0.1,0.3,0.5\}$ \\
Batch size & $\{16,32,64\}$ \\
\bottomrule
\end{tabular}
\end{table}

\begin{figure}[t]
    \centering
    \includegraphics[width=0.9\linewidth]{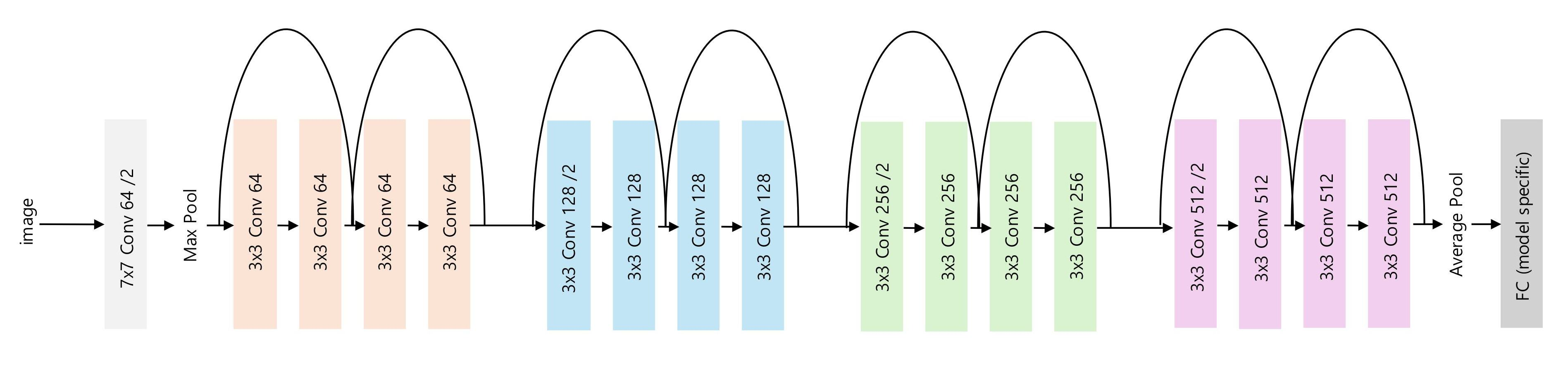}
    \caption{ResNet-18 backbone architecture used for medical imaging survival modeling. The original fully connected layer is removed, and model-specific heads are applied to the extracted feature vectors to produce survival predictions.}
    \label{fig:resnet18_backbone}
\end{figure}

\subsection{Computational Resources}
All experiments are reproducible on standard CPU or GPU hardware. Tabular experiments were run on CPU only and require a single node with at least 24~GB of system memory. Each hyperparameter optimization cell, corresponding to one method, one dataset, and one seed with 30 trials and up to 200 training epochs, takes approximately 5--30 minutes of wall-clock time depending on dataset size. The full tabular benchmark, consisting of 6 datasets, 9 methods, and 20 seeds, requires approximately 270 CPU-hours.

Image experiments require a single GPU with at least 20~GB of memory. We used NVIDIA L40S 44~GB and A100 80~GB nodes interchangeably. Each hyperparameter optimization cell, consisting of 20 trials and up to 200 training epochs, takes approximately 15-90 minutes depending on dataset size. The full image benchmark, consisting of 4 datasets, 9 methods, and 5 seeds, requires approximately 90 GPU-hours.

The synthetic simulation study can be run on a laptop-class CPU. In our experiments, the full simulation sweep was run on an Apple~M4 machine with 10 CPU cores and 32~GB RAM and completed in approximately 6-8 hours. SODEN with adjoint backpropagation dominated the simulation wall-clock time, requiring approximately 130 seconds per fit, whereas all other methods completed in approximately 1-7 seconds per fit.

Including preliminary experiments, repeated hyperparameter-configuration runs, and debugging runs that were not included in the final reported results, the total compute used for the project is estimated to be approximately 3--5 times the headline compute reported above.

\subsection{Deep Survival Analysis Models}\label{appendix:models}

\paragraph{CoxCC and CoxTime}
CoxCC and CoxTime extend the Cox proportional hazards framework by replacing the linear predictor with a neural network \citep{kvamme2019time}. CoxCC retains the proportional hazards assumption,
\begin{equation}
    \lambda(t | x) = \lambda_0(t)\exp\{g(x)\},
\end{equation}
where $\lambda_0(t)$ is the baseline hazard and $g(x)$ is a time-independent neural risk score. CoxTime relaxes this assumption by allowing the risk score to depend on time,
\begin{equation}
    \lambda(t | x) = \lambda_0(t)\exp\{g(x,t)\},
\end{equation}
thereby permitting time-varying covariate effects and crossing survival curves. We used the PyCox implementation: \texttt{https://github.com/havakv/pycox}.

To reduce the computational cost of evaluating full risk sets, these implementations use a case-control approximation to the Cox partial likelihood. For a subject $i$ who experiences an event at time $T_i$, the model compares the event subject against a sampled control set $J_i$ of individuals who remain at risk at $T_i$. The resulting loss is
\begin{equation}
    \mathcal{L} = \frac{1}{n}\sum_{i=1}^n \frac{1}{|J_i|}\sum_{j \in J_i}\log\left(1+\exp\{g(x_j)-g(x_i)\}\right).
\end{equation}
This softplus form penalizes the model when a control subject $j$ is assigned a higher risk score than the event subject $i$.

After training, a baseline hazard estimate is required to obtain absolute survival probabilities. Since the neural network predicts relative risk, the baseline hazard is recovered nonparametrically using the Breslow estimator. For a unique event time $t_m$ with $d_m$ events and risk set $\mathcal{R}(t_m)$, the CoxCC baseline hazard estimate is
\begin{equation}
    \hat{\lambda}_0(t_m) = \frac{d_m}{\sum_{j \in \mathcal{R}(t_m)} \exp\{g(x_j)\}}.
\end{equation}
For CoxTime, the risk score depends on the event time, so the network must be evaluated at $t_m$ for each subject in the risk set:
\begin{equation}
    \hat{\lambda}_0(t_m) = \frac{d_m}{\sum_{j \in \mathcal{R}(t_m)} \exp\{g(x_j,t_m)\}}.
\end{equation}
The cumulative baseline hazard $\hat{\Lambda}_0(t)$ is then obtained by summing these estimates over event times $t_m \leq t$.

\paragraph{DeepHit}
DeepHit formulates survival analysis as a discrete-time prediction problem \citep{lee2018deephit}. The time axis is divided into intervals, and the network outputs a probability mass function $\mathbf{p}=(p_1,\ldots,p_K)$, where $p_k$ is the predicted probability that the event occurs in interval $k$. The cumulative incidence function is
\begin{equation}
    F(t_k | x) = \sum_{j=1}^k p_j.
\end{equation}
We implemeted DeepHit based on the PyCox implementation: \texttt{https://github.com/havakv/pycox}.

The training objective combines a likelihood term and a ranking loss:
\begin{equation}
    \mathcal{L} = (1-\alpha)\mathcal{L}_{\mathrm{NLL}} + \alpha \mathcal{L}_{\mathrm{Rank}},
\end{equation}
where $\alpha$ controls the relative weight of the two components. For an uncensored subject whose event occurs in interval $k^*$, the likelihood contribution is $-\log p_{k^*}$. For a subject censored in interval $k^*$, the likelihood contribution is $-\log\{1-F(t_{k^*} | x)\}$. The ranking loss encourages subjects who experience earlier events to have higher predicted cumulative incidence at the corresponding event time:
\begin{equation}
    \mathcal{L}_{\mathrm{Rank}} = \sum_{A,B}\mathbb{I}(T_A<T_B)\exp\left(\frac{F(T_A | x_B)-F(T_A | x_A)}{\sigma}\right),
\end{equation}
where $\sigma$ is a smoothing parameter.

\paragraph{NnetSurv}
NnetSurv models survival in discrete time through interval-specific conditional event probabilities. Let $\lambda_j(x)$ denote the probability of an event in interval $[t_{j-1},t_j)$ conditional on survival up to $t_{j-1}$:
\begin{equation}
    \lambda_j(x) = P(T \in [t_{j-1},t_j) | T \geq t_{j-1}, x) = \sigma(f_j(x)).
\end{equation}
The corresponding survival probability at $t_k$ is
\begin{equation}
    S(t_k | x) = \prod_{j=1}^k \{1-\lambda_j(x)\}.
\end{equation}
The model is trained with a binary cross-entropy loss over all intervals in which each subject is at risk. For subject $i$ with observed interval index $k_i$ and event indicator $\delta_i$, the target is $y_{ij}=0$ for $j<k_i$ and $y_{ik_i}=\delta_i$. The loss is
\begin{equation}
    \mathcal{L} = -\sum_{i=1}^n\sum_{j=1}^{k_i}\left[y_{ij}\log \lambda_j(x_i) + (1-y_{ij})\log\{1-\lambda_j(x_i)\}\right].
\end{equation}

\paragraph{Survival Mixture Density Network}
The Survival Mixture Density Network (MDN) models the time-to-event distribution using a Gaussian mixture distribution on a latent time scale \citep{han2022survival}. We implemented MDN based on the authors' official codebase: \texttt{https://github.com/XintianHan/Survival-MDN}. To ensure positive event times, the observed time is represented as $T=\mathrm{softplus}(T^*)=\log(1+\exp\{T^*\})$, where the latent variable $T^*$ follows a conditional Gaussian mixture:
\begin{equation}
    f_{T^*}(u | x) = \sum_{k=1}^K \pi_k(x)\mathcal{N}\{u;\mu_k(x),\sigma_k(x)\}.
\end{equation}
The density of the observed event time follows from the change-of-variables formula:
\begin{equation}
    f_T(t | x) = f_{T^*}\{\mathrm{softplus}^{-1}(t) | x\}\left|\frac{d}{dt}\mathrm{softplus}^{-1}(t)\right| = f_{T^*}\{\mathrm{softplus}^{-1}(t) | x\}\frac{\exp(t)}{\exp(t)-1}.
\end{equation}
The survival function is
\begin{equation}
    S_T(t | x) = 1 - F_{T^*}\{\mathrm{softplus}^{-1}(t) | x\},
\end{equation}
where $F_{T^*}$ is the CDF of the latent Gaussian mixture. The model is trained by minimizing the censored negative log-likelihood,
\begin{equation}
    \mathcal{L} = -\sum_{i=1}^n\left[\delta_i \log f_T(t_i | x_i) + (1-\delta_i)\log S_T(t_i | x_i)\right].
\end{equation}

\paragraph{DeSurv}
DeSurv parameterizes the conditional distribution function through a monotone neural construction \citep{danks2022desurv}. Rather than modeling the hazard directly, DeSurv defines a non-negative function
\begin{equation}
    g(t,x) = \frac{\partial u(t,x)}{\partial t}
\end{equation}
using a softplus output, and constructs
\begin{equation}
    u(t,x) = \int_0^t g(s,x) ds.
\end{equation}
In our implementation, this integral is evaluated using fixed Gauss--Legendre quadrature with 15 nodes. The conditional CDF is then obtained by mapping $u(t,x)$ to $[0,1]$:
\begin{equation}
    F(t | x) = \tanh\{u(t,x)\}.
\end{equation}
The event density follows by differentiation:
\begin{equation}
    f(t | x) = \frac{\partial F(t | x)}{\partial t} = \left[1-F(t | x)^2\right]g(t,x).
\end{equation}
Thus,
\begin{equation}
    S(t | x) = 1-F(t | x), \qquad \lambda(t | x) = \frac{f(t | x)}{S(t | x)} = \{1+F(t | x)\}g(t,x).
\end{equation}
The hazard is therefore a derived quantity rather than the primitive parameterized object. DeSurv is trained by minimizing the standard censored negative log-likelihood,
\begin{equation}
    \mathcal{L} = -\frac{1}{n}\sum_{i=1}^n\left[\delta_i \log f(t_i | x_i) + (1-\delta_i)\log S(t_i | x_i)\right].
\end{equation}

\paragraph{SODEN}
SODEN is a continuous-time survival model based on Neural Ordinary Differential Equations \citep{tang2022soden}. We implemented SODEN using the authors' codebase: \texttt{https://github.com/jiaqima/SODEN}. The model represents the instantaneous hazard $\lambda(t | x)$ with a neural network and computes the cumulative hazard through the ODE integral
\begin{equation}
    \Lambda(t | x) = \int_0^t \lambda(u | x) du.
\end{equation}
This integral is evaluated during the forward pass using a differentiable adaptive ODE solver. The model is trained by minimizing the censored negative log-likelihood,
\begin{equation}
    \mathcal{L} = -\frac{1}{n}\sum_{i=1}^n\left[\delta_i \log \lambda(t_i | x_i) - \Lambda(t_i | x_i)\right].
\end{equation}
If the input time scale is transformed during preprocessing, the corresponding Jacobian adjustment is included so that the likelihood remains valid on the original time scale.

\begin{figure}[!p]
    \centering
    \includegraphics[width=0.99\linewidth]{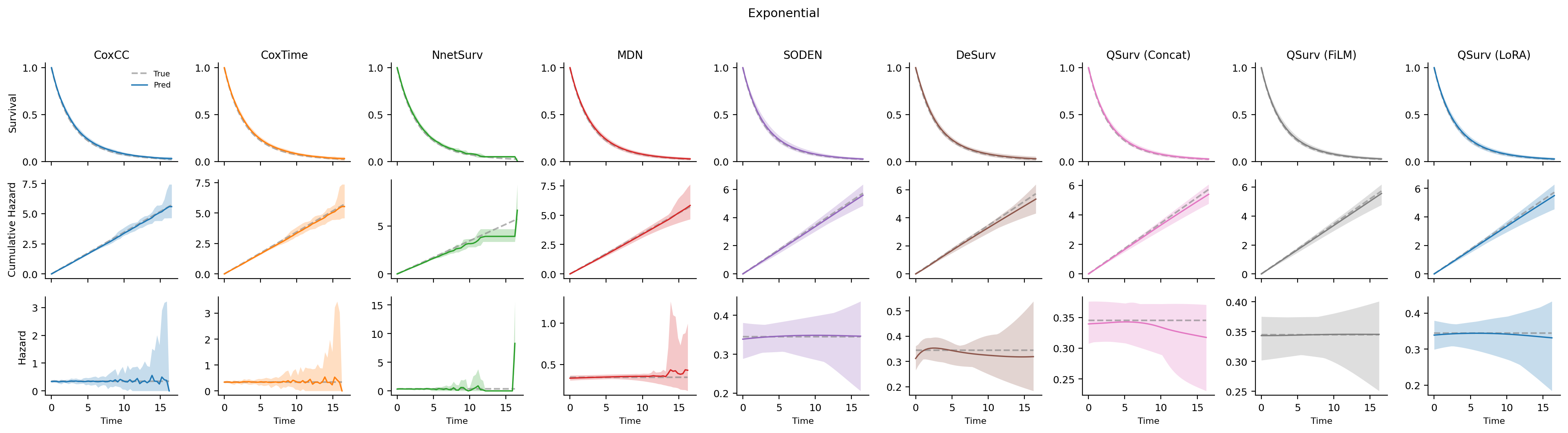}
    \caption{True vs. predicted survival, cumulative hazard, and instantaneous hazard functions for exponential distribution scenario. Shades are 95\% confidence interval.}
    \label{fig:sim_exponential}
\end{figure}

\begin{figure}[!p]
    \centering
    \includegraphics[width=0.99\linewidth]{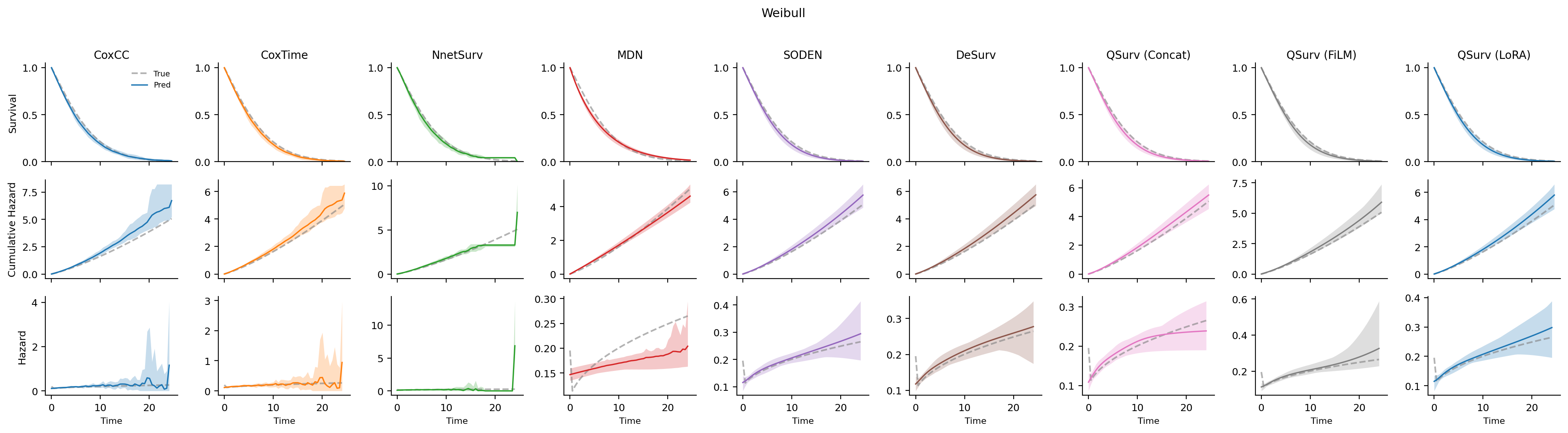}
    \caption{True vs. predicted survival, cumulative hazard, and instantaneous hazard functions for Weibull distribution scenario. Shades are 95\% confidence interval.}
    \label{fig:sim_weibull}
\end{figure}

\begin{figure}[!p]
    \centering
    \includegraphics[width=0.99\linewidth]{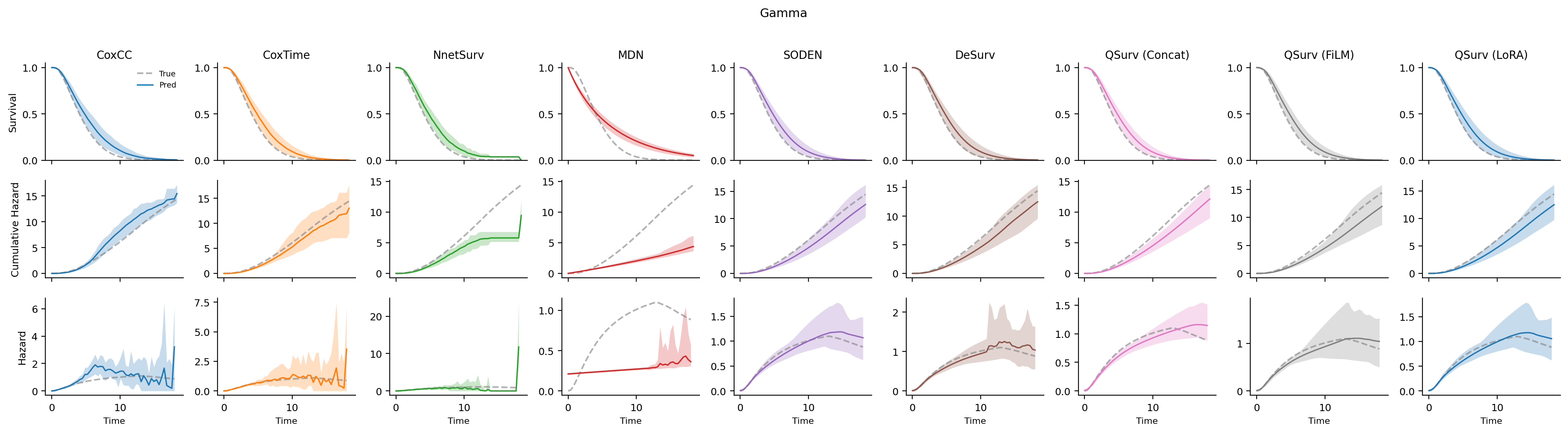}
    \caption{True vs. predicted survival, cumulative hazard, and instantaneous hazard functions for gamma distribution scenario. Shades are 95\% confidence interval.}
    \label{fig:sim_gamma}
\end{figure}

\begin{figure}[!p]
    \centering
    \includegraphics[width=0.99\linewidth]{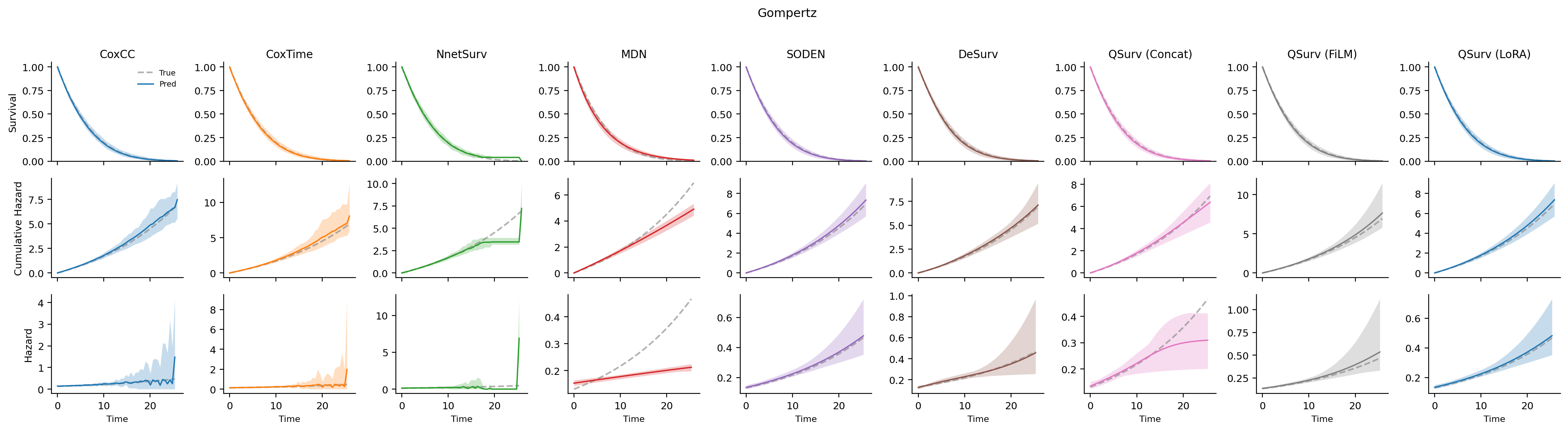}
    \caption{True vs. predicted survival, cumulative hazard, and instantaneous hazard functions for Gompertz distribution scenario. Shades are 95\% confidence interval.}
    \label{fig:sim_gompertz}
\end{figure}

\begin{figure}[!p]
\centering
    \includegraphics[width=0.99\linewidth]{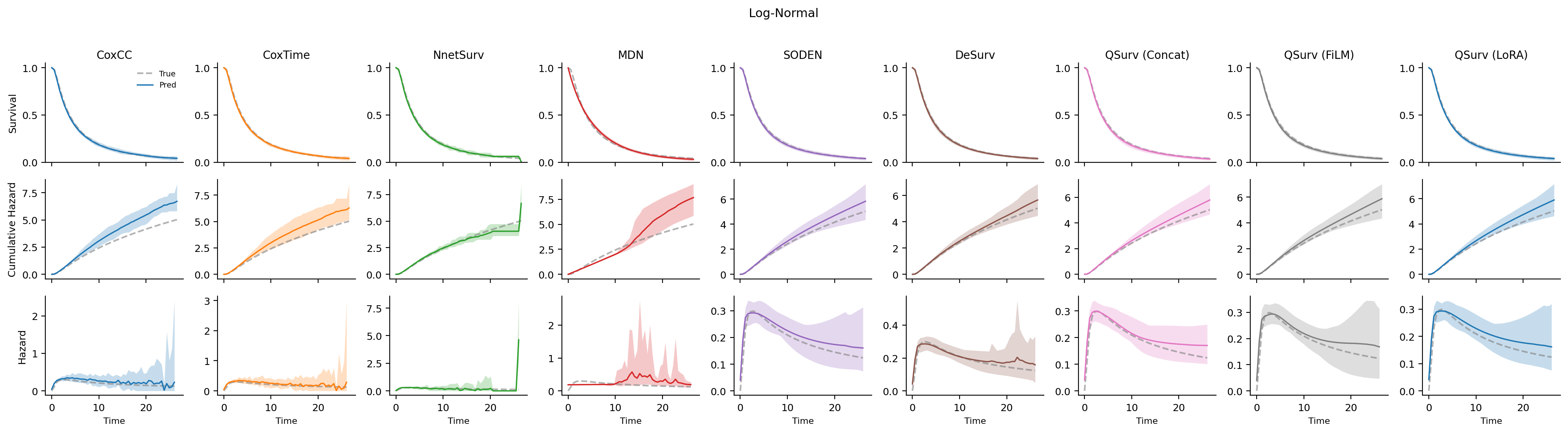}
    \caption{True vs. predicted survival, cumulative hazard, and instantaneous hazard functions for log-normal distribution scenario. Shades are 95\% confidence interval.}
    \label{fig:sim_lognormal}
\end{figure}

\begin{figure}[!p]
    \centering
    \includegraphics[width=0.99\linewidth]{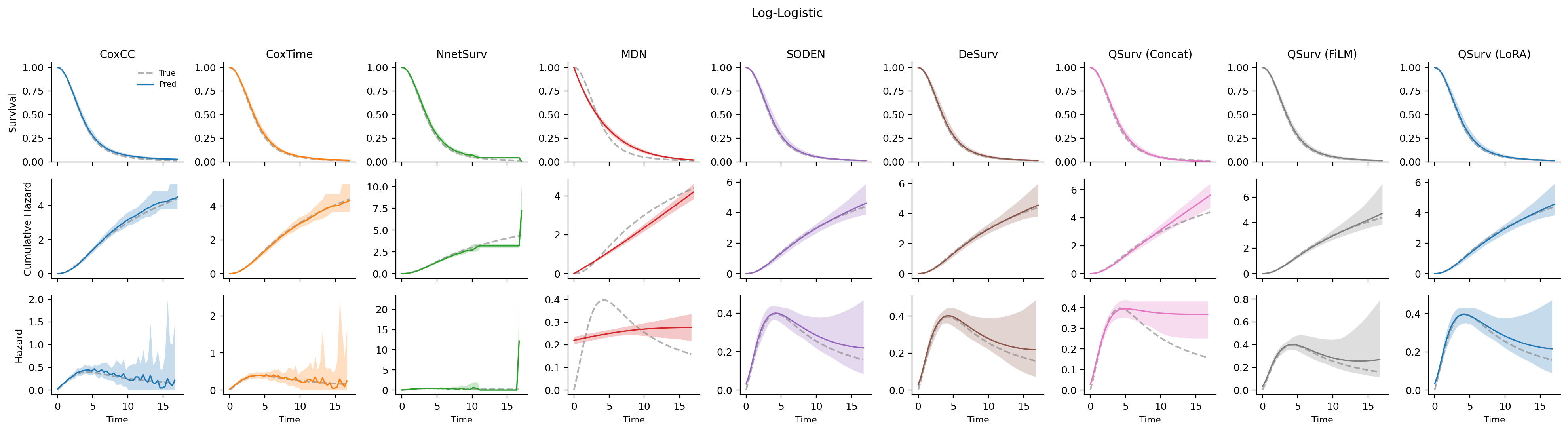}
    \caption{True vs. predicted survival, cumulative hazard, and instantaneous hazard functions for log-logistic distribution scenario. Shades are 95\% confidence interval.}
    \label{fig:sim_loglogistic}
\end{figure}



\end{document}